\documentclass[11pt,letterpaper]{article}
\usepackage{thmtools}
\usepackage{thm-restate}
\usepackage{bm}
\usepackage{mathrsfs}           %
\usepackage{vmargin,fancyhdr}   %
\usepackage{tikz}
\usetikzlibrary{positioning,chains,fit,shapes,calc}
\usepackage{algorithm}
\usepackage{algpseudocode}

\usepackage{subcaption}
\usepackage{enumerate}

\usepackage{amsmath,amssymb, amsthm}    %
\usepackage{verbatim}           %
\usepackage{xspace}             %
\usepackage{graphicx,float}     %
\usepackage{ifthen,calc}        %
\usepackage{textcomp}           %
\usepackage{fancybox}           %
\usepackage{hhline}             %
\usepackage{float}              %

\usepackage[vflt]{floatflt}     %
\usepackage[small,compact]{titlesec}
\usepackage{setspace}
\usepackage{color}
\definecolor{ForestGreen}{rgb}{0.1333,0.5451,0.1333}
\usepackage{thumbpdf}
\usepackage[letterpaper,
colorlinks,linkcolor=ForestGreen,citecolor=ForestGreen,
backref,
bookmarks,bookmarksopen,bookmarksnumbered]
{hyperref}

\setpapersize{USletter}
\setmarginsrb{1in}{.5in}        %
{1in}{.5in}        %
{.25in}{.25in}     %
{.25in}{.5in}      %
\newcommand{\showccc}[0]{0}
\newcommand{\ccc}[2][nothing]{%
	\ifthenelse{\showccc=0}{}{
		\ensuremath{^{\Lsh\Rsh}}\marginpar{\raggedright\tiny\textsf{%
				\ifthenelse{\equal{#1}{nothing}}{}{\textbf{#1}\\}#2}}}}
\newcounter{hours}\newcounter{minutes}
\newcommand{\hhmm}{%
	\setcounter{hours}{\time/60}%
	\setcounter{minutes}{\time-\value{hours}*60}%
	\ifthenelse{\value{hours}<10}{0}{}\thehours:%
	\ifthenelse{\value{minutes}<10}{0}{}\theminutes}
\ifthenelse{\showccc=0}{\rhead{}}{\rhead{\today \ [\hhmm]}}
\newtheorem{theorem}{Theorem}

\newtheorem{proposition}{Proposition}
\newtheorem{corollary}{Corollary}
\newtheorem{definition}{Definition}

\newtheorem{remark}{Remark}
\newtheorem{lemma}{Lemma}
\newtheorem{fact}{Fact}

\usepackage{float}
\floatstyle{plain}\newfloat{myfig}{t}{figs}[section]
\floatname{myfig}{\textsc{Figure}}
\floatstyle{plain}\newfloat{myalg}{H}{algs}[section]
\floatname{myalg}{}
\newcommand{\defeq}{\stackrel{\mathrm{{\scriptscriptstyle def}}}{=}}
\newcommand{\norm}[1]{\left\lVert#1\right\rVert}
\newcommand{\inprod}[2]{\left\langle#1, #2\right\rangle}

\newcommand{\argmin}{\textup{argmin}} 
\newcommand{\R}{\mathbb{R}}

\newcommand{\N}{\mathbb{N}}

\newcommand{\half}{\frac{1}{2}}
\newcommand{\thalf}{\tfrac{1}{2}}
\newcommand{\1}{\mathbf{1}}
\newcommand{\E}{\mathbb{E}}
\newcommand{\Var}{\textup{Var}}

\newcommand{\Nor}{\mathcal{N}}

\newcommand{\Oh}[1]{O\left(#1\right)}
\newcommand{\tOh}[1]{\tilde{O}\left(#1\right)}
\newcommand{\eps}{\epsilon}
\newcommand{\lam}{\lambda}
\newcommand{\xset}{\mathcal{X}}

\newcommand{\tvd}[2]{\norm{#1 - #2}_{\textup{TV}}}

\newcommand{\tx}{\tilde{x}}
\newcommand{\ty}{\tilde{y}}

\newcommand{\tran}{\mathcal{T}}
\newcommand{\prop}{\mathcal{P}}
\newcommand{\oracle}{\mathcal{O}}
\newcommand{\id}{\mathbf{I}}
\newcommand{\pih}{\hat{\pi}}
\newcommand{\pistart}{\pi^{\text{start}}}

\newcommand{\by}{\bar{y}}
\newcommand{\alg}{\mathcal{A}}
\newcommand{\hp}{\hat{p}}
\newcommand{\bya}{\bar{y}_\alpha}
\newcommand{\sya}{y^*_\alpha}
\newcommand{\bal}{\bar{\alpha}}
\newcommand{\hal}{\hat{\alpha}}

\newcommand{\zstart}{Z_{\textup{start}}}
\newcommand{\jac}{\mathbf{J}}

\newcommand{\csg}{\texttt{Composite-Sample}}
\newcommand{\cssm}{\texttt{Composite-Sample-Shared-Min}}
\newcommand{\sjd}{\texttt{Sample-Joint-Dist}}
\newcommand{\yor}{\texttt{Sample-Y}}
\newcommand{\covar}{\boldsymbol{\Sigma}}

\definecolor{burntorange}{rgb}{0.8, 0.33, 0.0}

\definecolor{pink}{cmyk}{0, 0.7808, 0.4429, 0.1412}

  \usepackage{nth}
  \usepackage{intcalc}

\usepackage{url}

\begin{document}

	\begin{titlepage}
		\def\thepage{}
		\thispagestyle{empty}
		
		\title{Composite Logconcave Sampling with a Restricted Gaussian Oracle} 
		
		\date{}
		\author{
			Ruoqi Shen\thanks{These authors contributed equally.}\\
			University of Washington \\
			{\tt shenr3@cs.washington.edu} 
			\and
			Kevin Tian\footnotemark[1] \\
			Stanford University \\
			{\tt kjtian@stanford.edu}
			\and
			Yin Tat Lee \\
			University of Washington and Microsoft Research \\
			{\tt yintat@uw.edu}
		}
		
		\maketitle

\abstract{
We consider sampling from composite densities on $\R^d$ of the form $d\pi(x) \propto \exp(-f(x) - g(x))dx$ for well-conditioned $f$ and convex (but possibly non-smooth) $g$, a family generalizing restrictions to a convex set, through the abstraction of a \emph{restricted Gaussian oracle}. For $f$ with condition number $\kappa$, our algorithm runs in $\Oh{\kappa^2 d \log^2\tfrac{\kappa d}{\eps}}$ iterations, each querying a gradient of $f$ and a restricted Gaussian oracle, to achieve total variation distance $\eps$. The restricted Gaussian oracle, which draws samples from a distribution whose negative log-likelihood sums a quadratic and $g$, has been previously studied \cite{CousinsV18, MouFWB19} and is a natural extension of the proximal oracle used in composite optimization. Our algorithm is conceptually simple and obtains stronger provable guarantees and greater generality than existing methods for composite sampling. We conduct experiments showing our algorithm vastly improves upon the hit-and-run algorithm for sampling the restriction of a (non-diagonal) Gaussian to the positive orthant.
}
 		
	\end{titlepage}

\section{Introduction}
\label{sec:intro}

We study the problem of approximately sampling from a distribution $\pi$ on $\R^d$, with density 
\begin{equation}\label{eq:compositesampling}\frac{d\pi(x)}{dx} \propto \exp\left(-f(x) - g(x)\right).\end{equation}
Here, $f: \R^d \rightarrow \R$ is assumed to be ``well-behaved'' (i.e.\ has finite condition number), and $g: \R^d \rightarrow \R$ is a convex, but possibly non-smooth function. This problem generalizes the special case of sampling from $\exp(-f(x))$ for well-behaved $f$, simply by setting $g$ to be uniformly zero. The existing (and extensive) literature on logconcave sampling, a natural problem family with roots in Bayesian statistics, machine learning, and theoretical computer science, typically focuses on the case when the log-density is well-behaved, and the distribution has support $\R^d$. Indeed, even the specialization of \eqref{eq:compositesampling} where $g$ indicates a convex set is not well-understood; existing bounds on mixing time for this restricted setting are large polynomials in $d$ \cite{BrosseDMP17, BubeckEL18}, and typically weaker than guarantees in the general logconcave setting \cite{LovaszV06a, LovaszV06b}, where no assumptions are made at all other than convexity of $f + g$, and only access to a zeroth order oracle is assumed\footnote{Throughout, we refer to a first order oracle for function $f$ as returning on query $x \in \R^d$, the pair $(f(x), \nabla f(x))$, whereas a zeroth order oracle only returns $f(x)$. Typical methods developed for sampling in the well-conditioned log-density regime are based on interacting with first order oracles.}. 

Sampling from logconcave distributions and optimization of convex functions have a close relationship, which has been extensively studied \cite{BertsimasV04, LovaszV06b}. However, the toolkit for first-order convex optimization has to date been much more flexible in terms of the types of problems it is able to handle, beyond optimizing well-conditioned functions. Examples of problem families which efficient first-order methods for convex optimization readily generalize to solving are
\[\min_{x \in \xset} f(x),\text{ where } \xset \subseteq \R^d \text{ is a convex set,}\]
as well as its generalization
\begin{equation}\label{eq:compositeopt}\min_{x \in \R^d} f(x) + g(x),\text{ where } g: \R^d \rightarrow \R \text{ is convex and admits a proximal oracle.}\end{equation}
The seminal work \cite{BeckT09} extends accelerated gradient methods to solve \eqref{eq:compositeopt} via proximal oracles, and has prompted many follow-up studies. Existence of an efficient proximal oracle is a natural measure of ``simplicity'' of $g$ in the context of composite optimization, which we now define.

\begin{definition}[Proximal oracle]
\label{def:proximaloracle}
$\oracle(\lam, v)$ is a \emph{proximal oracle} for convex $g: \R^d \rightarrow \R$ if it returns
\[\oracle(\lambda, v) \gets \argmin_{x \in \R^d}\left\{\frac{1}{2\lambda}\norm{x - v}_2^2 + g(x)\right\}.\]
\end{definition}
In other words, a proximal oracle minimizes functions which sum a quadratic and $g$. It is clear that the proximal oracle definition implies they can also handle arbitrary sums of linear functions and quadratics, as the resulting function can be rewritten as the sum of a constant and a single quadratic. Definition~\ref{def:proximaloracle} is desirable as many natural non-smooth composite objectives arising in learning settings, such as the Lasso \cite{Tibshirani96} and elastic net \cite{ZouH05}, admit efficient proximal oracles.

\subsection{Our contribution}
\label{ssec:contribution}

Motivated by the success of the proximal oracle framework, we study sampling from the family \eqref{eq:compositesampling} through the natural extension of Definition~\ref{def:proximaloracle}, which we term a ``restricted Gaussian oracle''. Informally, the oracle samples from a Gaussian (with covariance a multiple of $\id$) restricted by $g$.

\begin{definition}[Restricted Gaussian oracle]
\label{def:rgoracle}
$\oracle(\lam, v)$ is a \emph{restricted Gaussian oracle} for convex $g: \R^d \rightarrow \R$ if it returns
\[\oracle(\lambda, v) \gets \textup{sample from the distribution with density} \propto \exp\left(-\frac{1}{2\lambda}\norm{x - v}_2^2 - g(x)\right).\]
\end{definition}

The notion of a restricted Gaussian oracle has appeared previously \cite{CousinsV18, MouFWB19}, and its efficient implementation was a key subroutine in the fastest (zeroth-order) sampling algorithm for general logconcave distributions \cite{CousinsV18}. It was shown in \cite{MouFWB19} that a variety of composite distributions arising in practical applications, including coordinate-separable $g$, and $\ell_1$ or group Lasso regularized densities, admit such oracles. Our main result is an algorithm efficiently sampling from \eqref{eq:compositesampling}, assuming access to a restricted Gaussian oracle for $g$ and the minimizer $x^*$ of $f + g$.\footnote{This assumption is not restrictive, as efficient algorithms minimize $f + g$ in $\tOh{\sqrt{\kappa}}$ gradient queries to $f$ and proximal oracle queries to $g$ \cite{BeckT09}. Proximal oracle access is typically a weaker assumption than restricted Gaussian oracle access. We discuss effects of inexactness in this minimization procedure in Appendix~\ref{app:approximate}.}

\begin{theorem}
\label{thm:mainclaim}
Consider a distribution of the form \eqref{eq:compositesampling}, where $f$ has a condition number $\kappa$, and convex $g$ admits a restricted Gaussian oracle $\oracle$. Also, assume we know the minimizer $x^*$ of $f + g$. Algorithm~\ref{alg:csg}, $\csg$, samples from $\pi$ within total variation distance $\eps \in [0, 1]$, in $O(\kappa^2 d \log^2\tfrac{\kappa d}{\eps})$ iterations. Each iteration queries $\nabla f$ and $\oracle$ an expected constant number of times.
\end{theorem}

Recent work \cite{MouFWB19} also considered the problem of composite sampling via a restricted Gaussian oracle. However, their work also assumed access to the normalization constant of the restricted Gaussian, as well as Lipschitzness of $g$, amongst other criteria. Our result, Theorem~\ref{thm:mainclaim}, holds with no additional assumptions other than the relevant oracle access, including in the absence of a warm start. While there remains a gap between our runtime\footnote{Throughout, the $\tilde{O}$ notation hides logarithmic factors in $\kappa$, $d$, and $\eps^{-1}$.} of $\tOh{\kappa^2 d}$ and recent runtimes of $\tOh{\kappa d}$ in the non-composite setting \cite{LeeST20}, our algorithm substantially improves upon prior composite sampling work in both generality and runtime guarantees. We believe this provides evidence that the restricted Gaussian oracle is a useful abstraction in studying logconcave sampling with composite potentials. 

Finally, we remark that although our method follows several reductions, each is conceptually lightweight (as discussed in the following section) and easily implementable via either a rejection sampling procedure or oracle calls. To demonstrate this empirically, we evaluate our method for the task of sampling a (non-diagonal) Gaussian restricted to the positive orthant in Section~\ref{sec:experiments}. 

\subsection{Technical overview}
\label{ssec:techoverview}

We now survey the main components in the development of our algorithm. 

\textbf{Reduction to the shared minimizer case.} We first observe that we can without loss of generality assume that $f$ and $g$ share a minimizer. In particular, by shifting both functions by a linear term, i.e.\ $\tilde{f}(x) \defeq f(x) - \inprod{\nabla f(x^*)}{x}$, $\tilde{g}(x) \defeq g(x) + \inprod{\nabla f(x^*)}{x}$, where $x^*$ is the minimizer of $f + g$, first-order optimality implies both $\tilde{f}$ and $\tilde{g}$ are minimized by $x^*$. Moreover, implementation of a first-order oracle for $\tilde{f}$ and a restricted Gaussian oracle for $\tilde{g}$ are immediate without additional assumptions. This modification becomes crucial for our later developments, and we expect this simple observation, reminiscent of ``variance reduction'' techniques in stochastic optimization \cite{Johnson013}, to be broadly applicable to improving algorithms for the sampling problem induced by \eqref{eq:compositesampling}.

\textbf{Beyond Moreau envelopes: expanding the space.} A typical approach in convex optimization in handling non-smooth objectives $g$ is to instead optimize its \emph{Moreau envelope}, defined by
\begin{equation}\label{eq:moreauenvelope}g^\eta(y) \defeq \min_{x \in \R^d} \left\{g(x) + \frac{1}{2\eta}\norm{x - y}_2^2\right\}.\end{equation}
Intuitively, the envelope $g^\eta$ trades off function value with proximity to $y$; a standard exercise shows that $g^\eta$ is smooth (has a Lipschitz gradient), with smoothness depending on $\eta$, and moreover that computing gradients of $g^\eta$ is equivalent to calling a proximal oracle (Definition~\ref{def:proximaloracle}). It is natural to extend this idea to the composite sampling setting, e.g.\ via sampling from the density
\[\exp\left(-f(x) - g^\eta(x)\right).\]
However, a variety of complications prevent such strategies from obtaining rates comparable to their noncomposite, well-conditioned counterparts, including difficulty in bounding closeness of the resulting distribution, as well as bias in drift of the sampling process due to error in gradients.

Our approach departs from this smoothing strategy in a crucial way, inspired by Hamiltonian Monte Carlo (HMC) methods \cite{Kramers40, Neal11}. Hamiltonian Monte Carlo can be seen as a discretization of the ubiquitous Langevin dynamics, on an expanded space. In particular, discretizations of Langevin dynamics simulate the stochastic differential equation $\tfrac{dx_t}{dt} = -\nabla f(x_t) + \sqrt{2}\tfrac{dW_t}{dt}$, where $W_t$ is Brownian motion. HMC methods instead simulate dynamics on an extended space $\R^d \times \R^d$, via an auxiliary ``velocity'' variable which accumulates gradient information. This is sometimes interpreted as a discretization of the underdamped Langevin dynamics \cite{ChengCBJ18}. HMC often has desirable stability properties, and the strategy of expanding the dimension via an auxiliary variable has been used in algorithms obtaining the fastest rates in the well-conditioned logconcave sampling regime \cite{ShenL19, LeeST20}. Inspired by this phenomenon, we consider the density on $\R^d \times \R^d$
\begin{equation}\label{eq:expandspace}\frac{d\pih}{dz}(z) \defeq \exp\left(-f(y) - g(x) - \frac{1}{2\eta}\norm{x - y}_2^2\right) \text{ where } z = (x, y).\end{equation}
Due to technical reasons, the family of distributions we use in our final algorithms are of slightly different form than \eqref{eq:expandspace}, but this simplification is useful to build intuition. Note in particular that the form of \eqref{eq:expandspace} is directly inspired by \eqref{eq:moreauenvelope}, where rather than maximizing over $x$, we directly expand the space. The idea is that for small enough $\eta$ and a set on $x$ of large measure, smoothness of $f$ will guarantee that the marginal of \eqref{eq:expandspace} on $x$ will concentrate $y$ near $x$, a fact we make rigorous. To sample from \eqref{eq:compositesampling}, we then show that a rejection filter applied to a sample $x$ from the marginal of \eqref{eq:expandspace} will terminate in constant steps. Consequently, it suffices to develop a fast sampler for \eqref{eq:expandspace}.

\textbf{Alternating sampling with an oracle.} The form of the distribution \eqref{eq:expandspace} suggests a natural strategy for sampling from it: starting from a current state $(x_k, y_k)$, we iterate
\begin{enumerate}
	\item Sample $y_{k + 1} \sim \exp\left(-f(y) - \tfrac{1}{2\eta}\norm{x_k - y}_2^2\right)$.
	\item Sample $x_{k + 1} \sim \exp\left(-g(x) - \frac{1}{2\eta}\norm{x - y_{k + 1}}_2^2\right)$, via a restricted Gaussian oracle.
\end{enumerate}
When $f$ and $g$ share a minimizer, taking a first-order approximation in the first step, i.e.\ sampling $y_{k + 1} \sim \exp(-f(x_k) - \inprod{\nabla f(x_k)}{y - x_k} - \tfrac{1}{2\eta}\norm{y - x_k}_2^2)$, can be shown to be a generalization of the $\texttt{Leapfrog}$ step of Hamiltonian Monte Carlo updates. However, for $\eta$ very small (as in our setting), we observe that the first step itself reduces to the case of sampling from a distribution with constant condition number, which can be performed in $\tilde{O}(d)$ gradient calls by e.g.\ Metropolized HMC \cite{DwivediCW018, ChenDWY19, LeeST20}. Moreover, it is not hard to see that this ``alternating marginal'' sampling strategy preserves the stationary distribution exactly, so no filtering is necessary. Directly bounding the conductance of this random walk, for small enough $\eta$, leads to an algorithm running in $\tOh{\kappa^2 d^2}$ iterations, each calling a restricted Gaussian oracle once, and a gradient oracle for $f$ roughly $\tOh{d}$ times. This latter guarantee is by an appeal to known bounds \cite{ChenDWY19, LeeST20} on the mixing time in high dimensions of Metropolized HMC for a well-conditioned distribution, a property satisfied by the $y$-marginal of \eqref{eq:expandspace} for small $\eta$. 

\textbf{Stability of Gaussians under bounded perturbations.} To obtain our tightest runtime result, we use that $\eta$ is chosen to be much smaller than $L^{-1}$ to show structural results about distributions of the form \eqref{eq:expandspace}, yielding tighter concentration for bounded perturbations of a Gaussian (i.e.\ the Gaussian has covariance $\tfrac{1}{\eta}\id$, and is restricted by $L$-smooth $f$ for $\eta \ll L^{-1}$). To illustrate, let
\[\frac{d\prop_x(y)}{dy} \propto \exp\left(-f(y) - \frac{1}{2\eta}\norm{y - x}_2^2\right)\]
and let its mean and mode be $\by_x$, $y^*_x$. It is standard that $\norm{\by_x - y^*_x}_2 \le \sqrt{d\eta}$, by $\eta^{-1}$-strong logconcavity of $\prop_x$. Informally, we show that for $\eta \ll L^{-1}$ and $x$ not too far from the minimizer of $f$, we can improve this to $\norm{\by_x - y^*_x}_2 = O(\sqrt{\eta})$; see Proposition~\ref{prop:min_perturb} for a precise statement. 

Using our structural results, we sharpen conductance bounds, improve the warmness of a starting distribution, and develop a simple rejection sampling scheme for sampling the $y$ variable in expected constant gradient queries. These improvements lead to our main result, an algorithm running in $\tOh{\kappa^2d}$ iterations. Our proofs are continuous in flavor and based on gradually perturbing the Gaussian and solving a differential inequality; we believe they may of independent interest.

\subsection{Related work}

The broad problem of sampling from a logconcave distribution (with no assumptions beyond convexity on the log-density) has attracted much interest in the theoretical computer science community, as it generalizes uniform sampling from a convex set. General bounds under zeroth-order query access imply logconcave distributions are samplable in polynomial time ($\tOh{d^4}$ in the absence of a warm start \cite{LovaszV06b}). For more densities with more favorable structure, however, the first-order access model is attractive to exploit said structure.

Since seminal work of \cite{Dalalyan17}, an exciting research direction has studied first-order random walks for distributions with well-behaved log-densities, developing guarantees under assumptions such as Lipschitz derivatives of different orders \cite{ChengCBJ18, DalalyanR18, ChenV19, ChenDWY19, DwivediCW018, DurmusM19, DurmusMM19, LeeSV18, MouMWBJ19, ShenL19, LeeST20}. To our knowledge, when the log-density $f$ has a condition number of $\kappa$ (with no other assumptions), to obtain $\eps$ total variation distance the best-known guarantee is $\tOh{\kappa d}$ calls to $\nabla f$ \cite{LeeST20}, and to obtain $\eps D$ $2$-Wasserstein distance\footnote{$D = \sqrt{d/\mu}$ is the scale-invariant \emph{effective diameter} of a $\mu$-strongly logconcave distribution.} the best-known is $\tOh{\kappa^{7/6} \eps^{-1/3} + \kappa \eps^{-2/3}}$ oracle calls \cite{ShenL19}. These results do not typically generalize beyond when the support of $f$ is $\R^d$, prompting study of a more flexible distribution family.

Towards this goal, recent works studied sampling from densities of the form \eqref{eq:compositesampling}, or its specializations (e.g.\ restrictions to a convex set). Several \cite{Pereyra16, BrosseDMP17, Bernton18} are based on Moreau envelope or proximal regularization strategies, and demonstrate efficiency under more stringent assumptions on the structure of the composite function $g$, but under minimal assumptions obtain fairly large provable mixing times $\Omega(d^5)$. Algorithms derived from proximal regularization have also been considered for non-composite sampling \cite{Wibisono19}. Another discretization strategy based on projections was studied by \cite{BubeckEL18}, but obtained mixing time $\Omega(d^7)$. Finally, improved algorithms for special constrained sampling problems have been proposed, such as simplex restrictions \cite{HsiehKRC18}.

Of particular relevance and inspiration to this work is the algorithm of \cite{MouFWB19}. By generalizing and adapting Metropolized HMC algorithms of \cite{DwivediCW018, ChenDWY19}, adopting a Moreau envelope strategy, and using (a stronger version of) the restricted Gaussian oracle access model, \cite{MouFWB19} obtained a runtime which in the best case scales as $\tOh{\kappa^2 d}$, similar to our guarantee. However, this result required a variety of additional assumptions, such as access to the normalization factor of restricted Gaussians, Lipschitzness of $g$, warmness of the start, and various problem parameter tradeoffs. The general problem of sampling from \eqref{eq:compositesampling} under minimal assumptions more efficiently than general-purpose logconcave algorithms is to the best of our knowledge unresolved (even under restricted Gaussian oracle access), a novel contribution of our method and mixing time bound.

\subsection{Roadmap}

Section~\ref{sec:algorithm} states our algorithm and subroutines, and provides a proof of Theorem~\ref{thm:mainclaim} assuming various properties of our process. We demonstrate empirical performance of our method in Section~\ref{sec:experiments}. We defer proofs of technical ingredients to appendices, but give strategy overviews in the body. 	%

\section{Preliminaries}
\label{sec:prelims}

\paragraph{General notation.} For $d \in \N$, $[d]$ denotes the set of naturals $1 \le i \le d$. We use the Loewner order $\preceq$ on symmetric matrices, $\id$ to denote the identity matrix of appropriate dimension, and $\norm{\cdot}_2$ to mean the Euclidean norm. $\Nor(\mu, \boldsymbol\Sigma)$ is the Gaussian density with specified mean and covariance.

\paragraph{Functions.}  We call differentiable $f: \R^d \rightarrow \R$ $L$-smooth if it has a Lipschitz gradient, i.e. $\norm{\nabla f(x) - \nabla f(y)}_2 \le L\norm{x - y}_2$ for all $x, y \in \R^d$. If $f$ is twice-differentiable, it is well-known this implies for all $x \in \R^d$, $\nabla^2 f(x) \preceq L\id$. We say twice-differentiable $f$ is strongly convex if $\mu \id \preceq \nabla^2 f(x)$ everywhere. When a function is $L$-smooth and $\mu$-strongly convex, we define its condition number $\kappa \defeq \tfrac{L}{\mu}$. Strong convexity and smoothness respectively imply for all $x, y \in \R^d$,
\[f(x) + \inprod{\nabla f(x)}{y - x} + \frac{\mu}{2}\norm{y - x}_2^2 \le f(y) \le f(x) + \inprod{\nabla f(x)}{y - x} + \frac{L}{2}\norm{y - x}_2^2.\]

\paragraph{Distributions.} We say distribution $\pi$ is logconcave if $\tfrac{d\pi}{dx}(x) = \exp(-f(x))$, for some convex function $f$; it is $\mu$-strongly logconcave if its negative log-density is $\mu$-strongly convex. It is known that $\mu$-strong logconcavity implies $\mu$-sub-Gaussian tails (e.g. \cite{DwivediCW018}, Lemma 1). For $A \subseteq \R^d$, $\pi(A) \defeq \int_{x \in A} d\pi(x)$; we denote the complement $\R^d \setminus A$ by $A^c$. We say distribution $\rho$ is $\beta$-warm with respect to $\pi$ if $\tfrac{d\rho(x)}{d\pi(x)} \le \beta$ everywhere. The total variation distance between two distributions $\pi$ and $\rho$ is $\tvd{\pi}{\rho} \defeq \sup_{A \subseteq \R^d} |\pi(A) - \rho(A)|$. Finally, for a density $\pi$ on $\R^d$ and function $h: \R^d \rightarrow \R$, 
\[\E_\pi[h(x)] \defeq \int h(x) d\pi(x),\; \Var_\pi[h(x)] \defeq \E_\pi\left[(h(x))^2\right] - \left(\E_\pi[h(x)]\right)^2.\]

\section{Algorithm}
\label{sec:algorithm}

In this section, we state the components of our method. Throughout, fix distribution $\pi$ with density 
\begin{equation}\label{eq:pidef}\begin{aligned}\frac{d\pi}{dx}(x) \propto \exp\left(-f(x) - g(x)\right),\text{where }f:\R^d \rightarrow \R\text{ is }L\text{-smooth, } \mu\text{-strongly convex,} \\
\text{and }g: \R^d \rightarrow \R \text{ admits a restricted Gaussian oracle } \oracle.\end{aligned}\end{equation}
Observe that distribution $\pi$ is $\mu$-strongly logconcave. We assume that we have precomputed $x^* \defeq \argmin_{x \in \R^d}\left\{f(x) + g(x)\right\}$; see discussion in Section~\ref{ssec:contribution}. Our algorithm proceeds in stages following the outline in Section~\ref{ssec:techoverview}, which are put together in Section~\ref{ssec:proofmainclaim} to prove Theorem~\ref{thm:mainclaim}.
\begin{enumerate}
	\item $\csg$ is reduced to $\cssm$, which takes as input a distribution with negative log-density $f + g$, where $f$ and $g$ share a minimizer; this reduction is given in Section~\ref{ssec:sharedmin}, and the remainder of the paper handles the shared-minimizer case.
	\item The algorithm $\cssm$ is a rejection sampling scheme built on top of sampling from a joint distribution $\pih$ on $(x, y) \in \R^d \times \R^d$ whose $x$-marginal approximates $\pi$. We give this reduction in Section~\ref{ssec:outerloop}. 
	\item The bulk of our analysis is for $\sjd$, an alternating marginal sampling algorithm for sampling from $\pih$. To implement marginal sampling, it alternates calls to $\oracle$ and a rejection sampling algorithm $\yor$. We prove its correctness in Section~\ref{ssec:alternate}. 
\end{enumerate}

\begin{algorithm}[ht!]\caption{$\csg(\pi, x^*, \eps)$}
	\label{alg:csg}
	\textbf{Input:} Distribution $\pi$ of form \eqref{eq:pidef}, $x^*$ minimizing negative log-density of $\pi$, $\eps \in [0, 1]$. \\
	\textbf{Output:} Sample $x$ from a distribution $\pi'$ with $\tvd{\pi'}{\pi} \le \eps$.
	\begin{algorithmic}[1]
		\State $\tilde{f}(x) \gets f(x) - \inprod{\nabla f(x^*)}{x}$, $\tilde{g}(x) \gets g(x) + \inprod{\nabla f(x^*)}{x}$
		\State \Return $\cssm(\pi, \tilde{f}, \tilde{g}, x^*, \eps)$
	\end{algorithmic}
\end{algorithm}

\begin{algorithm}[ht!]\caption{$\cssm(\pi, f, g, x^*, \eps)$}
	\label{alg:cssm}
	\textbf{Input:} Distribution $\pi$ of form \eqref{eq:pidef}, where $f$ and $g$ are both minimized by $x^*$, $\eps \in [0, 1]$. \\
	\textbf{Output:} Sample $x$ from a distribution $\pi'$ with $\tvd{\pi'}{\pi} \le \eps$.
	\begin{algorithmic}[1]
		\While {\textbf{true}}
		\State Define the set
		\begin{equation}\label{eq:omegadef} \Omega \defeq \left\{x \mid \norm{x - x^*}_2 \le 4\sqrt{\frac{d\log(288\kappa/\eps)}{\mu}}\right\}\end{equation}
		\State $x \gets \sjd(f, g, x^*, \oracle, \tfrac{\eps}{18})$
		\If{$x \in \Omega$}
		\State $\tau \sim \text{Unif}[0, 1]$
		\State $y \gets \yor(f, x, \eta)$
		\State $\alpha \gets \exp\left(f(y) - \inprod{\nabla f(x)}{y - x} - \tfrac{L}{2}\norm{y - x}_2^2 + g(x) + \tfrac{\eta L^2}{2}\norm{x - x^*}_2^2\right)$
		\State $\hat{\theta} \gets \exp\left(-f(x) - g(x) + \tfrac{\eta}{2(1 + \eta L)}\norm{\nabla f(x)}_2^2\right)(1 + \eta L)^{\frac{d}{2}}\alpha$
		\If{$\tau \le \tfrac{\hat{\theta}}{4}$}
		\State \Return $x$
		\EndIf
		\EndIf
		\EndWhile 
	\end{algorithmic}
\end{algorithm}
\begin{algorithm}[ht!]\caption{$\sjd(f, g, x^*, \eta, \oracle, \delta)$}
	\label{alg:sjd}
	\textbf{Input:} $f$, $g$ of form \eqref{eq:pidef} both minimized by $x^*$, $\delta \in [0, 1]$, $\eta > 0$, $\oracle$ restricted Gaussian oracle for $g$.\\
	\textbf{Output:} Sample $x$ from a distribution $\pih'$ with $\tvd{\pih'}{\pih} \le \delta$, where we overload $\pih$ to mean the marginal of \eqref{eq:pihdef} on the $x$ variable.
	\begin{algorithmic}[1]
		\State $\eta \gets \tfrac{1}{32 L\kappa d\log(16\kappa/\delta)}$
		\State Let $\pih$ be the density with
		\begin{equation}\label{eq:pihdef}\frac{d\pih}{dx}(z) \propto \exp\left(-f(y) - g(x) - \frac{1}{2\eta}\norm{y - x}_2^2 - \frac{\eta L^2}{2}\norm{x - x^*}_2^2\right) \end{equation}
		\State Call $\oracle$ to sample $x_0 \sim \pistart$, for
		\begin{equation}\label{eq:pistartdef}\frac{d\pistart(x)}{dx} \propto \exp\left(-\frac{L + \eta L^2}{2}\norm{x - x^*}_2^2 - g(x)\right)\end{equation}
		\State $K \gets \frac{2^{26}\cdot100}{\eta\mu}\log\left(\frac{d\log(16\kappa)}{4\delta}\right)$ (see Remark~\ref{rem:comments})
		\For{$k \in [K]$}
		\State Call $\yor\left(f, x_{k - 1}, \eta, \tfrac{\delta}{2Kd\log(\frac{d\kappa}{\delta})}\right)$ to sample $y_k \sim \pi_{x_{k - 1}}$ (Algorithm~\ref{alg:yor}), for
		\begin{equation}\label{eq:pixdef}\frac{d\pi_x}{dy}(y) \propto \exp\left(-f(y) - \frac{1}{2\eta}\norm{y - x}_2^2\right)\end{equation}
		\State Call $\oracle$ to sample $x_k \sim \pi_{y_k}$, for
		\begin{equation}\label{eq:piydef}\frac{d\pi_y}{dx}(x) \propto \exp\left(-g(x) - \frac{1}{2\eta}\norm{y - x}_2^2 - \frac{\eta L^2}{2}\norm{x - x^*}_2^2\right)\end{equation}
		\EndFor
		\State \Return $x_K$
	\end{algorithmic}
\end{algorithm}

\subsection{Reduction from $\csg$ to $\cssm$}
\label{ssec:sharedmin}

Correctness of $\csg$ is via the following properties, whose proofs are in Appendix~\ref{app:deferred}.

\begin{restatable}{proposition}{restatecsgcorrectness}
\label{prop:csgcorrectness}
Let $\tilde{f}$ and $\tilde{g}$ be defined as in $\csg$. 
\begin{enumerate}
	\item The density $\propto \exp(-f(x) - g(x))$ is the same as the density $\propto \exp(-\tilde{f}(x) - \tilde{g}(x))$.
	\item Assuming first-order (function and gradient evaluation) access to $f$, and restricted Gaussian oracle access to $g$, we can implement the same accesses to $\tilde{f}$, $\tilde{g}$ with constant overhead.
	\item $\tilde{f}$ and $\tilde{g}$ are both minimized by $x^*$.
\end{enumerate}
\end{restatable}

\subsection{Reduction from $\cssm$ to $\sjd$}
\label{ssec:outerloop}

$\cssm$ is a rejection sampling scheme, which accepts samples from subroutine $\sjd$ in the high-probability region $\Omega$ defined in \eqref{eq:omegadef}. We give a general analysis for approximate rejection sampling in Appendix~\ref{sssec:approxreject}, and Appendix~\ref{sssec:pipih} bounds relationships between distributions $\pi$ and $\pih$, defined in \eqref{eq:pidef} and \eqref{eq:pihdef} respectively (i.e.\ relative densities and normalization constant ratios). Combining these pieces proves the following main claim.

\begin{restatable}{proposition}{restatecssmcorrectness}
\label{prop:cssmcorrectness}
Let $\eta = \tfrac{1}{32 L\kappa d\log(288\kappa/\eps)}$, and assume $\sjd(f, g, x^*, \oracle, \delta)$ samples within $\delta$ total variation of the $x$-marginal on \eqref{eq:pihdef}. $\cssm$ outputs a sample within total variation $\eps$ of \eqref{eq:pidef} in an expected $O(1)$ calls to $\sjd$.
\end{restatable}

\subsection{Implementing $\sjd$}
\label{ssec:alternate}

$\sjd$ alternates between sampling marginals in the joint distribution $\pih$, as seen by definitions \eqref{eq:pixdef}, \eqref{eq:piydef}. In Appendix~\ref{sssec:sjdcorrect}, we give a short proof that marginal sampling attains the correct stationary distribution. We bound the conductance of the induced walk on iterates $\{x_k\}$ by combining an isoperimetry bound with a total variation guarantee between transitions of nearby points in Appendix~\ref{sssec:sjdconduct}. Finally, we give a simple rejection sampling scheme $\yor$ as Algorithm~\ref{alg:yor} for implementing the step \eqref{eq:pixdef}. Since the $y$-marginal of $\pih$ is a bounded perturbation of a Gaussian (intuitively, $f$ is $L$-smooth and $\eta^{-1} \gg L$), we show in a high probability region that rejecting from the sum of a first-order approximation to $f$ and the Gaussian succeeds in $2$ iterations. 

\begin{remark}\label{rem:comments}
	For simplicity of presentation, we were conservative in bounding constants throughout; in practice (cf.\ Section~\ref{sec:experiments}), we found that the constant in Line 4 is orders of magnitude too large (a constant $< 10$ sufficed). Several constants were inherited from prior analyses, which we do not rederive to save on redundancy.
\end{remark}

We now give a complete guarantee on the complexity of $\sjd$.

\begin{restatable}{proposition}{restatesjdguarantee}\label{prop:sjdguarantee}
$\sjd$ outputs a point with distribution within $\delta$ total variation distance from the $x$-marginal of $\pih$. The expected number of gradient queries per iteration is constant.
\end{restatable}

\subsection{Putting it all together: Proof of Theorem~\ref{thm:mainclaim}}
\label{ssec:proofmainclaim}

We show Theorem~\ref{thm:mainclaim} follows from the guarantees of Propositions~\ref{prop:csgcorrectness},~\ref{prop:cssmcorrectness}, and~\ref{prop:sjdguarantee}. By observing the value of $K$ in $\sjd$, we see that the number of total iterations in each call to $\sjd$ is bounded by $ O\left(\kappa^2 d \log^2\left(\tfrac{\kappa d}{\delta}\right)\right).$
Proposition~\ref{prop:sjdguarantee} also shows that every iteration, we require an expected constant number of gradient queries and calls to $\oracle$, the restricted Gaussian oracle for $g$, and that the resulting distribution has $\delta$ total variation from the desired marginal of $\pih$. Next, Proposition~\ref{prop:cssmcorrectness} implies that the number of calls to $\sjd$ in a run of $\cssm$ is bounded by a constant, the choice of $\delta$  is $\Theta(\eps)$, and the resulting point has total variation $\eps$ from the original distribution $\pi$. Finally, Proposition~\ref{prop:csgcorrectness} shows sampling from a general distribution of the form \eqref{eq:compositesampling} is reducible to one call of $\cssm$, and the requisite oracles are implementable. 	%

\section{Experiments}
\label{sec:experiments}
We test our algorithm on the problem of sampling from a Gaussian restricted to an orthant. Formally, for a Gaussian with mean $m$ and covariance $\covar$, and where $O$ is a random orthant\footnote{This generalizes the case of the positive orthant by changing signs of $m$ appropriately.} 
(coordinatewise sign restrictions on $\R^d$), we consider sampling from the distribution\footnote{The indicator $\1_{x\in O}$ is $0$ if $x \in O$ and $\infty$ otherwise.}
\begin{align*}
\pi^*(x) \sim \exp\left(-\frac{1}{2}(x-m)^\top \covar^{-1} (x-m) -\1_{x\in O}\right).
\end{align*}
This problem is motivated by applications in posterior estimation with side information that the variable of interest has sign constraints, e.g.\ in physics simulations \cite{norgaard2018application}. For such distributions with nondiagonal covariances, sampling in the high-dimensional regime can be challenging, and to our knowledge no high-accuracy practical samplers exist for this fundamental problem. 

We verify the correctness of our algorithm by using the output of na\"ive rejection sampling (accepting samples in $O$) on Gaussian distributions with random covariance and random mean in low dimensions, where we can meaningfully plot histograms. We defer this test to Appendix~\ref{app:experiment_details}.

In high dimensions, we show our algorithm vastly improves upon the hit-and-run method \cite{LovaszV06a}, the most efficient general-purpose logconcave sampler in practice. Hit-and-run has a mixing time of $O(d^3)$ theoretically \cite{LovaszV06a} and $O(d^2)$ empirically. We test our algorithm on randomly generated Gaussian distributions with dense covariance matrices. For fair comparison to hit-and-run (which works on well-rounded distributions), the condition numbers $\kappa$ of all randomly generated Gaussian distributions are small constants $\approx 10$ and the smoothness parameters are $\approx 5$. The main tunable parameter in our algorithm is the step size $\eta$, which we chose so that both $\yor$ and $\sjd$ reject with probability at most $\thalf$.

In Figure~\ref{fig:depend}, we compare the mixing times of our algorithm and hit-and-run and show the dependence on the dimension $d$. The mixing criterion used was that the process has an effective sample size $\text{ESS} > 10$ for all coordinates. To ensure a stable scaling of the mixing time, we use fixed mean $m = 0$. Our algorithm used step size $\eta \approx \tfrac{0.3}{d}$ for each $d = 20, 35, 50, 65, 80$. We show that our algorithm improves upon hit-and-run by a factor $O(d)$, which corroborates our theoretical analysis. 

In Figure~\ref{fig:auto}, we plot the autocorrelation of the two algorithms' trajectories for $d= 500$, projected on a random unit direction. We show that in the very high-dimensional regime, our algorithm can converge significantly faster than hit-and-run. In this experiment, each coordinate of $m$ is chosen uniformly at random from $[-0.5, 0.5]$, and $\eta = 0.0014$ for our algorithm. We include an autocorrelation plot of shorter trajectories showing mixing time of our algorithm in Appendix~\ref{app:experiment_details}.

\begin{figure}[ht!]
	\centering
	\begin{subfigure}{.44\textwidth}
		\centering
		\includegraphics[width=0.93\linewidth]{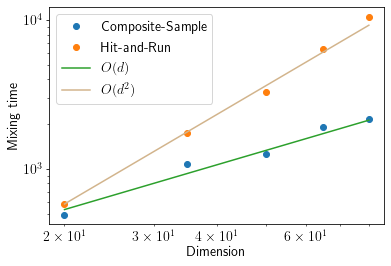}
		\captionof{figure}{Mixing time versus dimension. }
		\label{fig:depend}
	\end{subfigure}%
	\begin{subfigure}{.44\textwidth}
		\centering
		\includegraphics[width=1\linewidth]{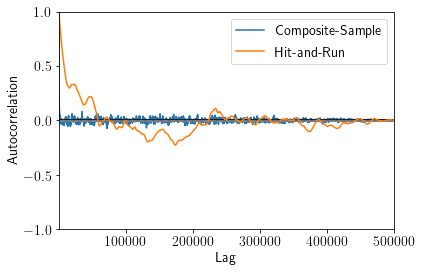}
		\captionof{figure}{Autocorrelation plot.}
		\label{fig:auto}
	\end{subfigure}
	\caption{Comparison between $\csg$ and Hit-and-Run. (a) Dimension dependence of mixing time, averaged over $10$ runs. (b) Autocorrelation plots of algorithms for $d = 500$. }
	\label{fig:test}
\end{figure}

	\subsection*{Acknowledgments}
	
	We thank Yair Carmon for suggesting the experiment in Section~\ref{sec:experiments}.
	\newpage
	\bibliographystyle{alpha}	
	\bibliography{composite-sampling}
	\newpage
	\begin{appendix}	

\section{Deferred proofs from Section~\ref{sec:algorithm}}
\label{app:deferred}

\subsection{Technical facts}
\label{ssec:technicalfacts}

We will repeatedly use the following facts throughout this paper.

\begin{fact}[Gaussian integral]\label{fact:gaussiandensity}
	For any $\lambda \ge 0$ and $v \in \R^d$,
	\[\int \exp\left(-\frac{1}{2\lambda}\norm{x - v}_2^2\right)dx = \left(2\pi\lambda\right)^{\frac{d}{2}}.\]
\end{fact}

\begin{fact}[\cite{Harge04}, Theorem 1.1]
	\label{fact:convexshrink}
	Let $\pi$ be a $\mu$-strongly logconcave density. Let $d\gamma_\mu(x)$ be the Gaussian density with covariance matrix $\mu^{-1}\id$. For any convex function $h$,
	\[\E_\pi[h(x - \E_\pi[x])] \le \E_{\gamma_\mu}[h(x - \E_{\gamma_\mu}[x])]. \]
\end{fact}

\begin{fact}[\cite{DwivediCW018}, Lemma 1]\label{fact:rbound}
	Let $\pi$ be a $\mu$-strongly logconcave distribution, and let $x^*$ minimize its negative log-density. Then, for $x \sim \pi$ and any $\delta \in [0, 1]$, with probability at least $1 - \delta$,
	\begin{equation}\label{eq:rbound}\norm{x - x^*}_2 \le \sqrt{\frac{d}{\mu}}\left(2 + 2\max\left(\sqrt[4]{\frac{\log(1/\delta)}{d}}, \sqrt{\frac{\log(1/\delta)}{d}} \right)\right).\end{equation}
\end{fact}

\begin{fact}[\cite{DurmusM19}, Theorem 1]\label{fact:distxstar}
	Let $\pi$ be a $\mu$-strongly logconcave distribution, and let $x^*$ minimize its negative log-density. Then, $\E_\pi[\norm{x - x^*}_2^2] \le \tfrac{d}{\mu}$.
\end{fact}

\subsection{Deferred proofs from Section~\ref{ssec:sharedmin}}

\restatecsgcorrectness*
\begin{proof}
	For $f$ and $g$ with properties as in \eqref{eq:pidef}, with $x^*$ minimizing $f + g$, define the functions 
	\[\tilde{f}(x) \defeq f(x) - \inprod{\nabla f(x^*)}{x},\; \tilde{g}(x) \defeq g(x) + \inprod{\nabla f(x^*)}{x},\]
	and observe that $\tilde{f}+ \tilde{g} = f + g$ everywhere. This proves the first claim. Further, implementation of a first-order oracle for $\tilde{f}$ and a restricted Gaussian oracle for $\tilde{g}$ are immediate assuming a first-order oracle for $f$ and a restricted Gaussian oracle for $g$, showing the second claim; any quadratic shifted by a linear term is the sum of a quadratic and a constant. We now show $\tilde{f}$ and $\tilde{g}$ have the same minimizer. By strong convexity, $\tilde{f}$ has a unique minimizer; first-order optimality shows that
	\[\nabla \tilde{f}(x^*) = \nabla f(x^*) - \nabla f(x^*) = 0,\]
	so this unique minimizer is $x^*$. Moreover, optimality of $x^*$ for $f + g$ implies that for all $x \in \R^d$,
	\[\inprod{\partial g(x^*) + \nabla f(x^*)}{x^* - x} \le 0.\]
	Here, $\partial g$ is a subgradient. This shows first-order optimality of $x^*$ for $\tilde{g}$ also, so $x^*$ minimizes $\tilde{g}$. 
\end{proof}

\subsection{Deferred proofs from Section~\ref{ssec:outerloop}}

\subsubsection{Approximate rejection sampling}
\label{sssec:approxreject}
We first define the rejection sampling framework we will use, and prove various properties.

\begin{definition}[Approximate rejection sampling]
	\label{def:reject_approx}
	Let $\pi$ be a distribution, with $\tfrac{d\pi}{dx}(x) \propto p(x)$. Suppose set $\Omega$ has $\pi(\Omega) = 1 - \eps'$, and distribution $\pih$ with $\frac{d\pih}{dx}(x) \propto \hp(x)$ has for some $C \ge 1$,
	\[\frac{p(x)}{\hp(x)} \le C \text{ for all } x \in \Omega, \text{ and } \frac{\int \hp(x) dx}{\int p(x) dx} \le 1.\]
	Suppose there is an algorithm $\alg$ which draws samples from a distribution $\pih'$, such that $\tvd{\pih'}{\pih} \le 1 - \delta$. We call the following scheme approximate rejection sampling: repeat independent runs of the following procedure until a point is outputted.
	\begin{enumerate}
		\item Draw $x$ via $\alg$ until $x \in \Omega$.
		\item With probability $\tfrac{p(x)}{C\hat{p}(x)}$, output $x$.
	\end{enumerate}
\end{definition}

\begin{lemma}\label{lem:reject_approx_proof}
	Consider an approximate rejection sampling scheme with relevant parameters defined as in Definition~\ref{def:reject_approx}, with $2\delta \le \tfrac{1 - \eps'}{C}$. The algorithm terminates in at most
	\begin{equation}\label{eq:algcalls}\frac{1}{\frac{1 - \eps'}{C} - 2\delta}\end{equation}
	calls to $\alg$ in expectation, and outputs a point from a distribution $\pi'$ with $\tvd{\pi'}{\pi} \le \eps' + \frac{2\delta C}{1 - \eps'}$.
\end{lemma}
\begin{proof}
	Define for notational simplicity normalization constants $Z \defeq \int p(x) dx$ and $\hat{Z} \defeq \int \hp(x) dx$. First, we bound the probability any particular call to $\alg$ returns in the scheme:
	\begin{equation}\label{eq:returnprob}\begin{aligned}\int_{x \in \Omega} \frac{p(x)}{C\hp(x)}d\pih'(x) &\ge \int_{x \in \Omega} \frac{p(x)}{C\hp(x)}d\pih(x) - \left|\int_{x \in \Omega}\frac{p(x)}{C\hp(x)}(d\pih'(x) - d\pih(x))\right|\\
	&= \int_{x \in \Omega} \frac{Z}{C\hat{Z}} d\pi(x) - \left|\int_{x \in \Omega}\frac{p(x)}{C\hp(x)}(d\pih'(x) - d\pih(x))\right|\\
	&\ge \frac{1 - \eps'}{C} - \int_{x \in \Omega}|d\pih'(x) - d\pih(x)| \ge \frac{1 - \eps'}{C} - 2\delta.\end{aligned}\end{equation}
	The second line followed by the definitions of $Z$ and $\hat{Z}$, and the third followed by triangle inequality, the assumed lower bound on $Z/\hat{Z}$, and the total variation distance between $\pih'$ and $\pih$. By linearity of expectation and independence, this proves the first claim.
	
	Next, we claim the output distribution is close in total variation distance to the conditional distribution of $\pi$ restricted to $\Omega$. The derivation of \eqref{eq:returnprob} implies
	\begin{equation}\label{eq:rationotbad}\begin{aligned} \int_{x \in \Omega} \frac{p(x)}{C\hp(x)}d\pih(x)\ge \frac{1 - \eps'}{C},\; \left|\int_{x \in \Omega}\frac{p(x)}{C\hp(x)}(d\pih'(x) - d\pih(x))\right| \le 2\delta,
	\\
	\implies 1 - \frac{2\delta C}{1 - \eps'} \le \frac{\int_{x \in \Omega} \frac{p(x)}{C\hp(x)}d\pih'(x)}{\int_{x \in \Omega} \frac{p(x)}{C\hp(x)}d\pih(x)} \le 1 + \frac{2\delta C}{1 - \eps'}.\end{aligned}\end{equation}
	Thus, the total variation of the true output distribution from $\pi$ restricted to $\Omega$ is
	\begin{align*}
	\half \int_{x \in \Omega} \left|\frac{d\pi(x)}{1 - \eps'} - \frac{\frac{p(x)}{C\hp(x)}d\pih'(x)}{\int_{x \in \Omega} \frac{p(x)}{C\hp(x)}d\pih'(x)}\right| \\
	\le \half \int_{x \in \Omega} \left|\frac{d\pi(x)}{1 - \eps'} - \frac{\frac{p(x)}{C\hp(x)}d\pih'(x)}{\int_{x \in \Omega} \frac{p(x)}{C\hp(x)}d\pih(x)}\right| + \half \int_{x \in \Omega} \left|\frac{\frac{p(x)}{C\hp(x)}d\pih'(x)}{\int_{x \in \Omega} \frac{p(x)}{C\hp(x)}d\pih(x)} - \frac{\frac{p(x)}{C\hp(x)}d\pih'(x)}{\int_{x \in \Omega} \frac{p(x)}{C\hp(x)}d\pih'(x)}\right| \\
	\le \half \int_{x \in \Omega} \left|\frac{d\pi(x)}{1 - \eps'} - \frac{\frac{p(x)}{C\hp(x)}d\pih'(x)}{\int_{x \in \Omega} \frac{p(x)}{C\hp(x)}d\pih(x)}\right|+ \frac{\delta C}{1 - \eps'} = \half \int_{x \in \Omega} \frac{d\pi(x)}{1 - \eps'}\left|1 - \frac{d\pih'}{d\pih}(x) \right| + \frac{\delta C}{1 - \eps'}.
	\end{align*}
	The first inequality was triangle inequality, and we bounded the second term by \eqref{eq:rationotbad}. To obtain the final equality, we used
	\begin{align*}\int_{x \in \Omega} \frac{p(x)}{C\hp(x)}d\pih(x) = \int_{x \in \Omega} \frac{Z}{C\hat{Z}}d\pi(x) = \frac{(1 - \eps')Z}{C\hat{Z}} \\
	\implies \frac{\frac{p(x)}{C\hp(x)}d\pih'(x)}{\int_{x \in \Omega} \frac{p(x)}{C\hp(x)}d\pih(x)} = \frac{p(x)}{Z} \cdot \frac{\hat{Z}}{\hp(x)} \cdot \frac{1}{1 - \eps'} \cdot d\pih'(x) = \frac{d\pi(x)}{1 - \eps'} \cdot \frac{d\pih'}{d\pih}(x).\end{align*}
	We now bound this final term. Observe that the given conditions imply that $\tfrac{d\pi}{d\pih}(x)$ is bounded by $C$ everywhere in $\Omega$. Thus, expanding we have
	\[\half \int_{x \in \Omega} \frac{d\pi(x)}{1 - \eps'}\left|1 - \frac{d\pih'}{d\pih}(x) \right|  \le \frac{C}{2(1 - \eps')} \int_{x \in \Omega} |d\pih(x) - d\pih'(x)| \le \frac{\delta C}{1 - \eps'}.\]
	Finally, combining these guarantees, and the fact that restricting $\pi$ to $\Omega$ loses $\eps'$ in total variation distance, yields the desired conclusion by triangle inequality.
\end{proof}

\begin{corollary}
	\label{corr:unbiased_reject_approx}
	Let $\hat{\theta}(x)$ be an unbiased estimator for $\tfrac{p(x)}{\hp(x)}$, and suppose $\hat{\theta}(x) \le C$ with probability 1 for all $x \in \Omega$. Then, implementing the procedure of Definition~\ref{def:reject_approx} with acceptance probability $\tfrac{\hat{\theta}(x)}{C}$ has the same runtime bound and total variation guarantee as given by Lemma~\ref{lem:reject_approx_proof}.
\end{corollary}
\begin{proof}
	It suffices to take expectations over the randomness of $\hat{\theta}$ everywhere in the proof of Lemma~\ref{lem:reject_approx_proof}.
\end{proof}

\subsubsection{Distribution ratio bounds}
\label{sssec:pipih}

We next show two bounds relating the densities of distributions $\pi$ and $\pih$.
We first define the normalization constants of \eqref{eq:pidef}, \eqref{eq:pihdef} for shorthand, and then tightly bound their ratio.
\begin{definition}[Normalization constants]
	\label{def:normconst}
	We denote normalization constants of $\pi$ and $\pih$ by
	\begin{align*}Z_\pi &\defeq \int_x \exp\left(-f(x) - g(x)\right) dx,\\
	Z_{\pih} &\defeq \int_{x, y} \exp\left(-f(y) - g(x) - \frac{1}{2\eta}\norm{y - x}_2^2 - \frac{\eta L^2}{2}\norm{x - x^*}_2^2\right)dxdy.\end{align*}
\end{definition}

\begin{lemma}[Normalization constant bounds]
	\label{lem:normratiobound}
	Let $Z_\pi$ and $Z_{\pih}$ be as in Definition \ref{def:normconst}.
	Then,
	\[\left(\frac{2\pi\eta}{1+\eta L}\right)^{\frac{d}{2}} \left(1+\frac{\eta L^{2}}{\mu}\right)^{-\frac{d}{2}} \le \frac{Z_{\pih}}{Z_{\pi}} \le (2\pi\eta)^{\frac{d}{2}}.\]
\end{lemma}
\begin{proof}
	For each $x$, by convexity we have
	\begin{equation}
	\label{eq:usefulbound}
	\begin{aligned}\int_y \exp\left(-f(y) - g(x) - \frac{1}{2\eta}\norm{y - x}_2^2 - \frac{\eta L^2}{2}\norm{x - x^*}_2^2\right)dy \\
	\le \exp\left(-g(x) - \frac{\eta L^2}{2}\norm{x - x^*}_2^2\right)\int_y\exp\left(-f(x)- \inprod{\nabla f(x)}{y - x} - \frac{1}{2\eta}\norm{y - x}_2^2\right)dy\\
	= \exp\left(-f(x) - g(x) - \frac{\eta L^2}{2}\norm{x - x^*}_2^2\right)\int_y \exp\left(\frac{\eta}{2}\norm{\nabla f(x)}_2^2 - \frac{1}{2\eta}\norm{y - x + \eta \nabla f(x)}_2^2\right) dy \\
	=(2\pi\eta)^{\frac{d}{2}}\exp\left(-f(x) - g(x)\right)\exp\left(\frac{\eta}{2}\norm{\nabla f(x)}_2^2 - \frac{\eta L^2}{2}\norm{x - x^*}_2^2\right) \\
	\le (2\pi\eta)^{\frac{d}{2}}\exp\left(-f(x) - g(x)\right).\end{aligned}
	\end{equation}
	Integrating both sides over $x$ yields the upper bound on $\tfrac{Z_{\pih}}{Z_{\pi}}$. Next, for the lower bound we have a similar derivation. For each $x$, by smoothness
	\begin{align*}
	\int_{y}\exp\left(-f(y)-g(x)-\frac{1}{2\eta}\left\Vert y-x\right\Vert_2 ^{2}-\frac{\eta L^{2}}{2}\left\Vert x-x^{*}\right\Vert_2^{2}\right)dy\\
	\geq \exp\left(-f(x)-g(x)-\frac{\eta L^{2}}{2}\left\Vert x-x^{*}\right\Vert_2^{2}\right)\int_{y}\exp\left(\left\langle \nabla f(x),x-y\right\rangle -\frac{1+\eta L}{2\eta}\left\Vert y-x\right\Vert_2^{2}\right)dy\\
	= \exp\left(-f(x)-g(x)-\frac{\eta L^{2}}{2}\left\Vert x-x^{*}\right\Vert ^{2} + \frac{\eta}{2(1+\eta L)}\left\Vert \nabla f(x)\right\Vert ^{2}\right)\left(\frac{2\pi\eta}{1+\eta L}\right)^{\frac{d}{2}}\\
	\geq  \exp\left(-f(x)-g(x)-\frac{\eta L^{2}}{2}\left\Vert x-x^{*}\right\Vert_2 ^{2}\right)\left(\frac{2\pi\eta}{1+\eta L}\right)^{\frac{d}{2}}.
	\end{align*}
	Integrating both sides over $x$ yields
	\begin{align*}
	\frac{Z_{\pih}}{Z_\pi} \ge \left(\frac{2\pi\eta}{1+\eta L}\right)^{\frac{d}{2}} \frac{\int_{x}\exp\left(-f(x)-g(x)-\frac{\eta L^{2}}{2}\left\Vert x-x^{*}\right\Vert_2 ^{2}\right)dx}{\int_{x}\exp\left(-f(x)-g(x)\right)dx}
	\geq \left(\frac{2\pi\eta}{1+\eta L}\right)^{\frac{d}{2}} \left(1+\frac{\eta L^{2}}{\mu}\right)^{-\frac{d}{2}}.
	\end{align*}
	The last inequality followed from Proposition~\ref{prop:normalizationratio}, where we used $f + g$ is $\mu$-strongly convex.
\end{proof}

\begin{lemma}[Relative density bounds]
	\label{lem:densityratio}
	Let $\eta = \tfrac{1}{32L\kappa d\log(288\kappa/\eps)}$. For all $x \in \Omega$, as defined in \eqref{eq:omegadef},  $\frac{d\pi}{d\pih}(x) \le 2$. Here, $\tfrac{d\pih}{dx}(x)$ denotes the marginal density of $\pih$. Moreover, for all $x \in \R^d$, $\frac{d\pi}{d\pih}(x) \ge \thalf$.
\end{lemma}
\begin{proof}
	We first show the upper bound. By Lemma~\ref{lem:normratiobound},
	\begin{equation}\label{eq:upperpartial}
	\begin{aligned}\frac{d\pi}{d\pih}(x) &= \frac{\exp\left(-f(x) - g(x)\right)}{\int_y\exp\left(-f(y) - g(x) - \frac{1}{2\eta}\norm{y - x}_2^2 - \frac{\eta L^2}{2}\norm{x - x^*}_2^2\right) dy} \cdot \frac{Z_{\pih}}{Z_\pi} \\
	&\le \frac{\exp\left(-f(x) - g(x)\right)}{\int_y\exp\left(-f(y) - g(x) - \frac{1}{2\eta}\norm{y - x}_2^2 - \frac{\eta L^2}{2}\norm{x - x^*}_2^2\right) dy} \cdot (2\pi\eta)^{\frac{d}{2}}.\end{aligned}\end{equation}
	We now bound the first term, for $x \in \Omega$. By smoothness, we have
	\[\frac{\exp\left(-f(y) - g(x)\right)}{\exp\left(-f(x) - g(x)\right)} \ge \exp\left(\inprod{\nabla f(x)}{x - y} - \frac{L}{2}\norm{y - x}_2^2\right), \]
	so applying this for each $y$,
	\begin{align*}
	\frac{\int_y \exp\left(-f(y) - g(x) - \frac{1}{2\eta}\norm{y - x}_2^2-\frac{\eta L^2}{2}\norm{x - x^*}_2^2\right)dy}{\exp\left(-f(x) - g(x)\right)} \\
	\ge \exp\left(-\frac{\eta L^2}{2}\norm{x - x^*}_2^2\right)\int_y \exp\left(\inprod{\nabla f(x)}{x - y} - \frac{1 + \eta L}{2\eta}\norm{y - x}_2^2\right)dy \\
	= \exp\left(-\frac{\eta L^2}{2}\norm{x - x^*}_2^2 + \frac{\eta}{2(1 + \eta L)}\norm{\nabla f(x)}_2^2\right)\int_y \exp\left(-\frac{1 + \eta L}{2\eta}\norm{x - y - \frac{\eta}{1 + \eta L}\nabla f(x)}_2^2\right)dy \\
	\ge \exp\left(-\frac{\eta L^2}{2}\cdot \frac{16d\log(288\kappa/\eps)}{\mu}\right)\left(\frac{2\pi\eta}{1 + \eta L}\right)^{\frac{d}{2}} \ge \frac{3}{4}\left(\frac{2\pi\eta}{1 + \eta L}\right)^{\frac{d}{2}}.
	\end{align*}
	In the last line, we used that $x \in \Omega$ implies $\norm{x - x^*}_2^2 \le \tfrac{16d\log(288\kappa/\eps)}{\mu}$, and the definition of $\eta$. Combining this bound with \eqref{eq:upperpartial}, we have the desired
	\[\frac{d\pi}{d\pih}(x) \le \frac{4}{3}\left(1 + \eta L\right)^{\frac{d}{2}} \le 2.\] Next, we consider the lower bound. By combining \eqref{eq:usefulbound} with Lemma~\ref{lem:normratiobound}, we have the desired
	\begin{align*}\frac{d\pi}{d\pih}(x) &= \frac{\exp\left(-f(x) - g(x)\right)}{\int_y\exp\left(-f(y) - g(x) - \frac{1}{2\eta}\norm{y - x}_2^2 - \frac{\eta L^2}{2}\norm{x - x^*}_2^2\right) dy} \cdot \frac{Z_{\pih}}{Z_\pi} \\
	&\ge (2\pi\eta)^{-\frac{d}{2}} \cdot \left(\frac{2\pi\eta}{1+\eta L}\right)^{\frac{d}{2}} \left(1+\frac{\eta L^{2}}{\mu}\right)^{-\frac{d}{2}} = \left(\frac{1}{1 + \eta L}\right)^{\frac{d}{2}}\left(1 + \eta L \kappa\right)^{-\frac{d}{2}} \ge \half.\end{align*}
\end{proof}

\subsubsection{Correctness of $\cssm$}
\restatecssmcorrectness*
\begin{proof}
	We remark that $\eta = \tfrac{1}{32 L\kappa d\log(288\kappa/\eps)}$ is precisely the choice of $\eta$ in $\sjd$ where $\delta = \eps/18$, as in $\cssm$. First, we may apply Fact~\ref{fact:rbound} to conclude that the measure of set $\Omega$ with respect to the $\mu$-strongly logconcave density $\pi$ is at least $1 - \eps/3$. The conclusion of correctness will follow from an appeal to Corollary~\ref{corr:unbiased_reject_approx}, with parameters
	\[C = 4,\; \eps' = \frac{\eps}{3},\; \delta = \frac{\eps}{18}.\]
	Note that indeed we have $\eps' + \tfrac{2\delta C}{1 - \eps'}$ is bounded by $\eps$, as $1 - \eps' \ge \tfrac{2}{3}$. Moreover, the expected number of calls \eqref{eq:algcalls} is clearly bounded by a constant as well. 
	
	We now show that these parameters satisfy the requirements of Corollary~\ref{corr:unbiased_reject_approx}. Define the functions
	\begin{align*}p(x) &\defeq \exp(-f(x) - g(x)),\\ \hp(x) &\defeq (2\pi\eta)^{-\frac{d}{2}}\int_y \exp\left(-f(y) - g(x) - \frac{1}{2\eta}\norm{y - x}_2^2 - \frac{\eta L^2}{2}\norm{x - x^*}_2^2\right)dy,\end{align*}
	and observe that clearly the densities of $\pi$ and $\pih$ are respectively proportional to $p$ and $\hp$. Moreover, define $Z = \int p(x) dx$ and $\hat{Z} = \int \hp(x) dx$. By comparing these definitions with Lemma~\ref{lem:normratiobound}, we have $Z = Z_\pi$ and $\hat{Z} = (2\pi\eta)^{-\frac{d}{2}}Z_{\pih}$, so by the upper bound in Lemma~\ref{lem:normratiobound}, $\hat{Z}/Z \le 1$. Next, we claim that the following procedure produces an unbiased estimator for $\tfrac{p(x)}{\hp(x)}$.
	\begin{enumerate}
		\item Sample $y \sim \pi_x$, where $\tfrac{d\pi_x(y)}{dy} \propto \exp\left(-f(y) - \tfrac{1}{2\eta}\norm{y - x}_2^2\right)$
		\item $\alpha \gets \exp\left(f(y) - \inprod{\nabla f(x)}{y - x} - \tfrac{L}{2}\norm{y - x}_2^2 + g(x) + \tfrac{\eta L^2}{2}\norm{x - x^*}_2^2\right)$
		\item Output $\hat{\theta}(x) \gets \exp\left(-f(x) - g(x) + \tfrac{\eta}{2(1 + \eta L)}\norm{\nabla f(x)}_2^2\right)(1 + \eta L)^{\frac{d}{2}}\alpha$
	\end{enumerate}
	To prove correctness of this estimator $\hat{\theta}$, define for simplicity
	\[Z_x \defeq \int_y \exp\left(-f(y) - g(x) - \frac{1}{2\eta}\norm{y - x}_2^2 - \frac{\eta L^2}{2}\norm{x - x^*}_2^2\right)dy.\]
	We compute, using $\tfrac{d\pi_x(y)}{dy} = \tfrac{\exp(-f(y) - g(x) - \frac{1}{2\eta}\norm{y - x}_2^2 - \frac{\eta L^2}{2}\norm{x - x^*}_2^2)}{Z_x}$, that
	\begin{align*}
	\E_{\pi_x}\left[\alpha\right] &= \int_y \exp\left(f(y) - \inprod{\nabla f(x)}{y - x} - \frac{L}{2}\norm{y - x}_2^2 + g(x) + \frac{\eta L^2}{2}\norm{x - x^*}_2^2\right)d\pi_x(y) \\
	&= \frac{1}{Z_x} \int_y \exp\left(- \inprod{\nabla f(x)}{y - x} - \frac{L}{2}\norm{y - x}_2^2 - \frac{1}{2\eta}\norm{y - x}_2^2\right) dy \\
	&= \frac{1}{Z_x}\exp\left(-\frac{\eta}{2(1 + \eta L)}\norm{\nabla f(x)}_2^2\right)\left(\frac{2\pi \eta}{1 + \eta L}\right)^{\frac{d}{2}}.
	\end{align*}
	This implies that the output quantity
	\[\hat{\theta}(x) = \exp\left(-f(x) - g(x) + \frac{\eta}{2(1 + \eta L)}\norm{\nabla f(x)}_2^2\right)(1 + \eta L)^{\frac{d}{2}}\alpha\]
	is unbiased for $\tfrac{p(x)}{\hp(x)} = \exp(-f(x) - g(x)) Z_x^{-1}(2\pi\eta)^{\frac{d}{2}}$. Finally, note that for any $y$ used in the definition of $\hat{\theta}(x)$, by using $f(y) - f(x) - \inprod{\nabla f(x)}{y - x} - \tfrac{L}{2}\norm{y - x}_2^2 \le 0$ via smoothness, we have
	\begin{align*}\hat{\theta}(x) &=  \exp\left(-f(x) - g(x) + \frac{\eta}{2(1 + \eta L)}\norm{\nabla f(x)}_2^2\right)(1 + \eta L)^{\frac{d}{2}}\alpha \\
	&\le (1 + \eta L)^{\frac{d}{2}}\exp\left(\frac{\eta}{2(1 + \eta L)}\norm{\nabla f(x)}_2^2 + \frac{\eta L^2}{2}\norm{x - x^*}_2^2\right) \\
	&\le (1 + \eta L)^{\frac{d}{2}}\exp\left(\eta L^2\norm{x - x^*}_2^2\right) \le 4.
	\end{align*}
	Here, we used the definition of $\eta$ and $L^2\norm{x - x^*}_2^2 \le 16L\kappa d\log(288\kappa/\eps)$ by the definition of $\Omega$. 
\end{proof}

\subsection{Deferred proofs from Section~\ref{ssec:alternate}}

Throughout this section, for error tolerance $\delta \in [0, 1]$ which parameterizes $\sjd$, we denote for shorthand a high-probability region $\Omega_\delta$ and its radius $R_\delta$ by
\begin{equation}\label{eq:omegadelta}\Omega_\delta \defeq \left\{x \mid \norm{x - x^*}_2 \le R_\delta\right\},\text{ for } R_\delta \defeq 4\sqrt{\frac{d\log(16\kappa/\delta)}{\mu}}.\end{equation}
The following density ratio bounds hold within this region, by simply modifying Lemma~\ref{lem:densityratio}.
\begin{corollary}
	\label{corr:densityratiodelta}
	Let $\eta = \tfrac{1}{32L\kappa d\log(16\kappa/\delta)}$, and let $\pih$ be parameterized by this choice of $\eta$ in \eqref{eq:pihdef}. For all $x \in \Omega_\delta$, as defined in \eqref{eq:omegadelta}, $\frac{d\pi}{d\pih}(x) \le 2$. Moreover, for all $x \in \R^d$, $\frac{d\pi}{d\pih}(x) \ge \half$.
\end{corollary}
The following claim follows immediately from applying Fact~\ref{fact:rbound}.
\begin{lemma}\label{lem:omegahighprob}
	With probability at least $1 - \tfrac{\delta^2}{8(1 + \kappa)^d}$, $x \sim \pih$ lies in $\Omega_\delta$.
\end{lemma}
Finally, when clear from context, we overload $\pih$ as a distribution on $x\in\R^d$ to be the $x$ component marginal of the distribution \eqref{eq:pihdef}, i.e. with density
\[ \frac{d\pih}{dx}(x) \propto \int_y\exp\left(-f(y) - g(x) - \frac{1}{2\eta}\norm{y - x}_2^2 - \frac{\eta L^2}{2}\norm{x - x^*}_2^2\right) dy.
\]
In Section~\ref{sssec:sjdcorrect}, we show $\pih$ is stationary for $\sjd$. In Section~\ref{sssec:sjdconduct}, we bound the \emph{conductance} of the walk, used in Section~\ref{sssec:runtime} to bound its mixing time and overall complexity. 

\subsubsection{Correctness of $\sjd$}
\label{sssec:sjdcorrect}
Correctness of $\sjd$ follows from the following simple lemma.

\begin{lemma}[Alternating marginal sampling]
	\label{lem:alternate_exact}
	Let $\pih$ be a density on two blocks $(x, y)$. Sample $(x, y) \sim \pih$, and then sample $\tx \sim \pih(\cdot, y)$, $\ty \sim \pih(\tx, \cdot)$. Then, the distribution of $(\tx, \ty)$ is $\pih$.
\end{lemma}
\begin{proof}
	The density of the resulting distribution at $(\tx, y)$ is proportional to the product of the (marginal) density at $y$ and the conditional distribution of $\tx \mid y$, which by definition is $\pih$. Therefore, $(\tx, y)$ is distributed as $\pih$, and the argument for $\ty$ follows symmetrically.
\end{proof}

\subsubsection{Conductance of $\sjd$}
\label{sssec:sjdconduct}
We bound the conductance of this random walk, as a process on the iterates $\{x_k\}$, to show the final point has distribution close to the marginal of $\pih$ on $x$. We use the well-known framework of bounding mixing time via \emph{average conductance}, introduced in \cite{LovaszK99}, and since extended by e.g.\ \cite{KannanLM06, GoelMT06, ChenDWY19}. We state a formulation by \cite{ChenDWY19} convenient for our purposes.

\begin{definition}[Restricted conductance]
	Let a random walk with stationary distribution $\pih$ on $x \in \R^d$ have transition densities $\tran_x$, and let $\Omega \subseteq \R^d$. The $\Omega$-restricted conductance, for $v \in (0, \thalf\pih(\Omega))$, is
	\[\Phi_{\Omega}(v) = \inf_{\pih(S \cap \Omega)\in(0, v]} \frac{\tran_S(S^c)}{\pih(S \cap \Omega)}, \text{ where } \tran_S(S^c) \defeq \int_{x \in S}\int_{x'\in S^c}\tran_x(x') d\pih(x)dx'.\]
\end{definition}

\begin{proposition}[Lemma 1, \cite{ChenDWY19}]
	\label{prop:mixviaconduct}
	Let $\pistart$ be a $\beta$-warm start for $\pih$, and let $x_0 \sim \pistart$. For some $\delta > 0$, let $\Omega \subseteq \R^d$ have $\pih(\Omega) \ge 1 - \tfrac{\delta^2}{2\beta^2}$. Suppose that a random walk with stationary distribution $\pih$ satisfies the $\Omega$-restricted conductance bound
	\[\Phi_{\Omega}(v) \ge \sqrt{B\log\left(\frac{1}{v}\right)},\text{ for all } v \in \left[\frac{4}{\beta},\half\right].\]
	Let $x_K$ be the result of $K$ steps of this random walk, starting from $x_0$. Then, for 
	\[K \ge \frac{64}{B}\log\left(\frac{\log\beta}{2\delta}\right),\]
	the resulting distribution of $x_K$ has total variation at most $\tfrac{\delta}{2}$ from $\pih$.
\end{proposition}

We state a well-known strategy for lower bounding conductance, via showing the stationary distribution has good \emph{isoperimetry} and that transition distributions of nearby points have large overlap.

\begin{proposition}[Lemma 2, \cite{ChenDWY19}]
	\label{prop:conductviaisotv}
	Let a random walk with stationary distribution $\pih$ on $x \in \R^d$ have transition distribution densities $\tran_x$, and let $\Omega \subseteq \R^d$, and let $\pih_\Omega$ be the conditional distribution of $\pih$ on $\Omega$. Suppose for any $x, x' \in \Omega$ with $\norm{x - x'}_2 \le \Delta$,
	\[\tvd{\tran_x}{\tran_{x'}} \le \half.\]
	Also, suppose $\pih_\Omega$ satisfies, for any partition $S_1$, $S_2$, $S_3$ of $\Omega$, where $d(S_1, S_2)$ is the minimum Euclidean distance between points in $S_1$, $S_2$, the log-isoperimetric inequality
	\begin{equation}\label{eq:logiso}\pih_\Omega(S_3) \ge \frac{1}{2\psi}d(S_1, S_2) \cdot \min\left(\pih_\Omega(S_1), \pih_\Omega(S_2)\right) \cdot \sqrt{\log\left(1 + \frac{1}{\min\left(\pih_\Omega(S_1), \pih_\Omega(S_2)\right)}\right)}.\end{equation}
	Then, we have the bound for all $v \in (0, \thalf]$
	\[\Phi_{\Omega}(v) \ge \frac{\Delta}{128\psi}\sqrt{\log\left(\frac{1}{v}\right)}.\]
\end{proposition}

To utilize Propositions~\ref{prop:mixviaconduct} and~\ref{prop:conductviaisotv}, we prove the following bounds in Appendices~\ref{ssec:warmstart},~\ref{ssec:tvclose}, and~\ref{ssec:isoperimetry}.

\begin{restatable}[Warm start]{lemma}{restatewarmstart}\label{lem:warmstart}
	For $\eta \le \tfrac{1}{L\kappa d}$, $\pistart$ defined in \eqref{eq:pistartdef} is a $2(1 + \kappa)^{\frac{d}{2}}$-warm start for $\pih$.
\end{restatable}

\begin{restatable}[Transitions of nearby points]{lemma}{restatetrantv}\label{lem:tv_closepts}
	Suppose $\eta L \le 1$, $\eta L^2R_{\delta}^2 \le \thalf$, and $400d^2\eta\le R_\delta^2$. For a point $x$, let $\tran_x$ be the density of $x_k$ after sampling according to Lines 6 and 7 of Algorithm~\ref{alg:sjd} from $x_{k - 1} = x$. For $x, x' \in \Omega_\delta$ with $\norm{x - x'}_2 \le \tfrac{\sqrt{\eta}}{10}$, for $\Omega_\delta$ defined in \eqref{eq:omegadelta}, we have $\tvd{\tran_x}{\tran_{x'}} \le \thalf$.
\end{restatable}

\begin{restatable}[Isoperimetry]{lemma}{restateiso}\label{lem:iso}
	Density $\pih$ and set $\Omega_\delta$ defined in \eqref{eq:pihdef}, \eqref{eq:omegadelta} satisfy \eqref{eq:logiso} with $\psi = 8\mu^{-\half}$.
\end{restatable}

We note that the parameters of Algorithm~\ref{alg:sjd} and the set $\Omega_\delta$ in \eqref{eq:omegadelta} satisfy all assumptions of Lemmas~\ref{lem:warmstart},~\ref{lem:tv_closepts}, and~\ref{lem:iso}. By combining these results in the context of Proposition~\ref{prop:conductviaisotv}, we see that the random walk satisfies the bound for all $v \in (0, \thalf]$:
\[\Phi_{\Omega_\delta}(v) \ge \sqrt{\frac{\eta\mu}{2^{20}\cdot100}\cdot\log\left(\frac{1}{v}\right)}.\]
Plugging this conductance lower bound, the high-probability guarantee of $\Omega_\delta$ by Lemma~\ref{lem:omegahighprob}, and the warm start bound of Lemma~\ref{lem:warmstart} into Proposition~\ref{prop:mixviaconduct}, we have the following conclusion.

\begin{corollary}[Mixing time of ideal $\sjd$]\label{corr:sjdmix}
	Assume that calls to $\yor$ are exact in the implementation of $\sjd$. Then, for any error parameter $\delta$, and
	\[K \defeq \frac{2^{26}\cdot100}{\eta\mu}\log\left(\frac{d\log(16\kappa)}{4\delta}\right),\]
	the distribution of $x_K$ has total variation at most $\tfrac{\delta}{2}$ from $\pih$.
\end{corollary}

\subsubsection{Complexity of $\sjd$}
\label{sssec:runtime}

We first state a guarantee on the subroutine $\yor$, which we prove in Appendix~\ref{ssec:yoracle}.

\begin{restatable}[$\yor$ guarantee]{lemma}{restateyor}\label{lem:yor} For $\delta \in [0, 1]$, define $R_\delta$ as in \eqref{eq:omegadelta}, and let $\eta = \tfrac{1}{32L\kappa d\log(16\kappa/\delta)}$. For any $x$ with $\norm{x - x^*}_2 \le \sqrt{\kappa d\log(16\kappa/\delta)}\cdot R_\delta$, Algorithm~\ref{alg:yor} ($\yor$) draws an exact sample $y$ from the density proportional to $\exp\left(-f(y) - \tfrac{1}{2\eta}\norm{y - x}_2^2\right)$ in an expected $2$ iterations.
\end{restatable}

We also state a result due to \cite{ChenDWY19}, which bounds the mixing time of 1-step Metropolized HMC for well-conditioned distributions; this handles the case when $\norm{x - x^*}_2$ is large in Algorithm~\ref{alg:yor}.

\begin{proposition}[Theorem 1, \cite{ChenDWY19}]\label{prop:atmostd}
	Let $\pi$ be a distribution on $\R^d$ whose negative log-density is convex and has condition number bounded by a constant. Then, Metropolized HMC from an explicit starting distribution mixes to total variation $\delta$ to the distribution $\pi$ in $O(d\log(\tfrac{d}{\delta}))$ iterations.
\end{proposition}

\restatesjdguarantee*
\begin{proof}
	Under exact implementation of $\yor$, Corollary~\ref{corr:sjdmix} shows the output distribution of $\sjd$ has total variation at most $\tfrac{\delta}{2}$ from $\pih$. Next, the resulting distribution of the subroutine $\yor$ is never larger than $\delta/(2Kd\log(\frac{d\kappa}{\delta}))$ in total variation distance away from an exact sampler. By running for $K$ steps, and using the coupling characterization of total variation, it follows that this can only incur additional error $\delta/(2d\log(\frac{d\kappa}{\delta}))$, proving correctness (in fact, the distribution is always at most $O((d\log(d\kappa/\delta))^{-1})$ away in total variation from an exact $\yor$).
	
	Next, we prove the guarantee on the expected gradient evaluations per iteration. Lemma~\ref{lem:yor} shows whenever the current iterate $x_k$ has $\norm{x - x^*}_2 \le \sqrt{\kappa d\log(16\kappa/\delta)} \cdot R_{\delta}$, the expected number of gradient evaluations is constant, and moreover Proposition~\ref{prop:atmostd} shows that the number of gradient evaluations is never larger than $O(d\log(\tfrac{d\kappa}{\delta}))$, where we use that the condition number of the log-density in \eqref{eq:pixdef} is bounded by a constant. Therefore, it suffices to show in every iteration $0 \le k \le K$, the probability $\norm{x_k - x^*}_2 > \sqrt{\kappa d\log(16\kappa/\delta)} \cdot R_{\delta}$ is $O((d\log(d\kappa/\delta))^{-1})$. By the warmness assumption in Lemma~\ref{lem:warmstart}, and the concentration bound in Fact~\ref{fact:rbound}, the probability $x_0$ does not satisfy this bound is negligible (inverse exponential in $\kappa d^2\log(\kappa/\delta)$). Since warmness is monotonically decreasing with an exact sampler\footnote{This fact is well-known in the literature, and a simple proof is that if a distribution is warm, then taking one step of the Markov chain induces a convex combination of warm point masses, and is thus also warm.}, and the accumulated error due to inexactness of $\yor$ is at most $O((d\log(d\kappa/\delta))^{-1})$ through the whole algorithm, this holds for all iterations.
\end{proof} 	%

\section{Mixing time ingredients}
\label{sec:ingredients}

We now prove facts which are used in the mixing time analysis of $\sjd$. Throughout this section, as in the specification of $\sjd$, $f$ and $g$ are functions with properties as in \eqref{eq:pidef}, and share a minimizer $x^*$.

\subsection{Warm start}
\label{ssec:warmstart}
We show that we obtain a warm start for the distribution $\pih$ in algorithm $\sjd$ via one call to the restricted Gaussian oracle for $g$, by proving Lemma~\ref{lem:warmstart}.

\restatewarmstart*
\begin{proof}
By the definitions of $\pih$ and $\pistart$ in \eqref{eq:pihdef}, \eqref{eq:pistartdef}, we wish to bound everywhere the quantity
\begin{equation}\label{eq:warmness}\frac{d\pistart}{d\pih}(x) = \frac{Z_{\pih}}{\zstart} \cdot  \frac{\exp\left(-\frac{L}{2}\norm{x - x^*}_2^2-\frac{\eta L^2}{2}\norm{x - x^*}_2^2 - g(x)\right)}{\int_y\exp\left(-f(y) - g(x) - \frac{1}{2\eta}\norm{y - x}_2^2 - \frac{\eta L^2}{2}\norm{x - x^*}_2^2\right) dy}. \end{equation}
Here, $Z_{\pih}$ is as in Definition~\ref{def:normconst}, and we let $\zstart$ denote the normalization constant of $\pistart$, i.e.
\[\zstart \defeq  \int_x \exp\left(-\frac{L}{2}\norm{x - x^*}_2^2 - \frac{\eta L^2}{2}\norm{x - x^*}_2^2- g(x)\right) dx.\]
Regarding the first term of \eqref{eq:warmness}, 
the earlier derivation \eqref{eq:usefulbound} showed
\[
\int_y \exp\left(-f(y) - g(x) - \frac{1}{2\eta}\norm{y - x}_2^2 - \frac{\eta L^2}{2}\norm{x - x^*}_2^2\right)dy 
\le (2\pi\eta)^{\frac{d}{2}}\exp\left(-f(x) - g(x)\right).
\]
Then, integrating, we can bound the ratio of the normalization constants
\begin{equation}
\label{eq:zhzstart}
\begin{aligned}
\frac{Z_{\pih}}{Z_{\pistart}}  &\leq \frac{ \int_x(2\pi\eta)^{\frac{d}{2}}\exp\left(-f(x) - g(x)\right)dx}{ \int_x \exp\left(-\frac{L}{2}\norm{x - x^*}_2^2 - \frac{\eta L^2}{2}\norm{x - x^*}_2^2- g(x)\right) dx}  \\
&\leq \frac{ \int_x(2\pi\eta)^{\frac{d}{2}}\exp\left(-f(x^*)-\frac{\mu}{2}\norm{x - x^*}_2^2 - g(x)\right)dx}{ \int_x \exp\left(-\frac{L}{2}\norm{x - x^*}_2^2 - \frac{\mu}{2}\norm{x - x^*}_2^2- g(x)\right) dx} \\
&\leq (2\pi\eta)^{\frac{d}{2}}\exp\left(-f(x^*)\right) \left( 1 + \frac{L}{\mu}\right)^{\frac{d}{2}}.
\end{aligned}
\end{equation}
The second inequality followed from $f$ is $\mu$-strongly convex and $\eta L^2 \leq \mu$ by assumption. The last inequality followed from Proposition \ref{prop:normalizationratio}, where we used  $\frac{\mu}{2}\norm{x - x^*}_2^2 + g(x)$ is $\mu$-strongly convex. Next, to bound the second term of \eqref{eq:warmness}, notice first that 
\begin{align*}
\frac{\exp\left(-\frac{L}{2}\norm{x - x^*}_2^2-\frac{\eta L^2}{2}\norm{x - x^*}_2^2 - g(x)\right)}{\int_y\exp\left(-f(y) - g(x) - \frac{1}{2\eta}\norm{y - x}_2^2 - \frac{\eta L^2}{2}\norm{x - x^*}_2^2\right) dy} 
= \frac{\exp \left(-\frac{L}{2}\norm{x - x^*}_2^2\right) }{\int_y\exp\left(-f(y)  - \frac{1}{2\eta}\norm{y - x}_2^2\right) dy}.
\end{align*}
It thus suffices to lower bound $\exp \left(\frac{L}{2}\norm{x - x^*}_2^2\right)\int_y\exp\left(-f(y)  - \frac{1}{2\eta}\norm{y - x}_2^2 \right) dy$. We have
\begin{equation}
\label{eq:warmnessotherterm}
\begin{aligned}
\exp \left(\frac{L}{2}\norm{x - x^*}_2^2\right)\int_y\exp\left(-f(y)  - \frac{1}{2\eta}\norm{y - x}_2^2 \right) dy \\
\geq \exp\left(-f(x)+\frac{ L}{2}\norm{x - x^*}_2^2  \right)\int_y\exp\left(- \langle \nabla f(x),y-x\rangle  - \left(\frac{1}{2\eta}+\frac{L}{2}\right)\norm{y-x}_2^2\right) dy \\
= \exp\left(-f(x)+\frac{ L}{2}\norm{x - x^*}_2^2  \right) \left( \frac{2\pi\eta }{1+L\eta}\right)^{\frac{d}{2}} \exp\left(\frac{\eta}{2(1+L\eta )}\norm{\nabla f(x)}_2^2\right) \\
\geq \exp(-f(x^*)) \left( \frac{2\pi\eta }{1+L\eta}\right)^{\frac{d}{2}}
\end{aligned}
\end{equation}
The first and third steps followed from $L$-smoothness of $f$, and the second applied the Gaussian integral (Fact~\ref{fact:gaussiandensity}). Combining the bounds in \eqref{eq:zhzstart} and \eqref{eq:warmnessotherterm}, \eqref{eq:warmness} becomes
\begin{align*}
\frac{d\pistart}{d\pih}(x) \leq  \left( 1 + \frac{L}{\mu}\right)^{\frac{d}{2}} \left( 1 + L\eta\right)^{\frac{d}{2}} \leq 2(1+\kappa)^{\frac{d}{2}},
\end{align*}
where $x \in \R^d$ was arbitrary, which completes the proof.
\end{proof}

\subsection{Transitions of nearby points}
\label{ssec:tvclose}

Here, we prove Lemma~\ref{lem:tv_closepts}. Throughout this section, $\tran_x$ is the density of $x_k$, according to the steps in Lines 6 and 7 of $\sjd$ (Algorithm~\ref{alg:sjd}) starting at $x_{k - 1} = x$. We also define $\prop_x$ to be the density of $y_k$, by just the step in Line 6. We first make a simplifying observation.

\begin{lemma}
\label{lem:tranleqprop}
For any two points $x$, $x'$, we have
\[\tvd{\tran_x}{\tran_{x'}} \le \tvd{\prop_x}{\prop_{x'}}.\]
\end{lemma}
\begin{proof}
This follows by the coupling characterization of total variation distance (see e.g. Chapter 5 of \cite{LevinPW09}). Per the optimal coupling of $y \sim \prop_x$ and $y' \sim \prop_{x'}$, whenever the total variation sets $y = y'$, we can couple the resulting distributions in Line 7 of $\sjd$ as well.
\end{proof}

Thus, it suffices to understand $\tvd{\prop_x}{\prop_{x'}}$ for nearby $x, x' \in \Omega_\delta$. Our proof of Lemma~\ref{lem:tv_closepts} combines two pieces: (1) bounding the ratio of normalization constants $Z_x$, $Z_{x'}$ of $\prop_x$ and $\prop_{x'}$ for nearby $x$, $x'$ in Lemma~\ref{lem:norm_ratio_closepts} and (2) the structural result Proposition~\ref{prop:min_perturb}. 
To bound the normalization constant ratio, we state two helper lemmas. Lemma~\ref{lem:ycloser} characterizes facts about the minimizer of 
\begin{equation}\label{eq:marginalfunc}f(y) + \frac{1}{2\eta}\norm{y - x}_2^2. \end{equation} 
\begin{lemma}
	\label{lem:ycloser}
	Let $f$ be convex with minimizer $x^*$, and $y_x$ minimize \eqref{eq:marginalfunc} for a given $x$. Then,
	\begin{enumerate}
		\item $\norm{y_x - y_{x'}}_2 \le \norm{x - x'}_2$.
		\item For any $x$, $\norm{y_x - x^*}_2 \le \norm{x - x^*}_2$.
		\item For any $x$ with $\norm{x - x^*}_2 \le R$, $\norm{x - y_x}_2 \le \eta LR$.
	\end{enumerate}
\end{lemma}
\begin{proof}By optimality conditions in the definition of $y_x$,
	\[\eta\nabla f(y_x) = x - y_x. \]
	Fix two points $x$, $x'$, and let $x_t \defeq (1 - t)x + tx'$. Letting $\jac_x(y_x)$ be the Jacobian matrix of $y_x$,
	\begin{align*}\frac{d}{dt} \eta\nabla f(y_{x_t}) = \frac{d}{dt}\left(x_t - y_{x_t}\right) &\implies \eta \nabla^2 f(y_{x_t}) \jac_x(y_{x_t}) (x' - x) = (\id - \jac_x(y_{x_t}))(x' - x)\\
	&\implies \jac_x(y_{x_t}) (x' - x) = (\id + \eta \nabla^2 f(y_{x_t}))^{-1}(x' - x).\end{align*}
	We can then compute
	\[y_{x'} - y_x = \int_0^1 \frac{d}{dt}y_{x_t} dt = \int_0^1 \jac_x(y_{x_t})(x' - x)dt = \int_0^1 (\id + \eta \nabla^2 f(y_{x_t}))^{-1}(x' - x)dt. \]
	By triangle inequality and convexity of $f$, the first claim follows:
	\[\norm{y_{x'} - y_x}_2 \le \int_0^1 \norm{(\id + \eta\nabla^2 f(y_{x_t}))^{-1}}_2 \norm{x' - x}_2 dt \le \norm{x' - x}_2. \]
	The second claim follows from the first by $y_{x^*} = x^*$. The third claim follows from the second via
	\[\norm{x - y_x}_2 = \eta\norm{\nabla f(y_x)}_2 \le \eta L\norm{y_x - x^*}_2 \le \eta LR.\]
\end{proof}
Next, Lemma~\ref{lem:gauint_in_closepts} states well-known bounds on the integral of a well-conditioned function $h$.
\begin{lemma}
	\label{lem:gauint_in_closepts}
	Let $h$ be a $L_h$-smooth, $\mu_h$-strongly convex function and let $y^*_h$ be its minimizer. Then
	\[\left(2\pi L_h^{-1}\right)^{\frac{d}{2}}\exp\left(-h(y^*_h)\right) \le \int_y \exp\left(-h(y)\right) \le \left(2\pi \mu_h^{-1}\right)^{\frac{d}{2}}\exp\left(-h(y^*_h)\right).\]
\end{lemma}
\begin{proof}
	By smoothness and strong convexity,
	\[\exp\left(-h(y^*_h) - \frac{L_h}{2}\norm{y - y^*_h}_2^2\right) \le \exp(-h(y)) \le \exp\left(-h(y^*_h) - \frac{\mu_h}{2}\norm{y - y^*_h}_2^2\right).\]
	The result follows by Gaussian integrals, i.e.\ Fact~\ref{fact:gaussiandensity}.
\end{proof}

We now define the normalization constants of $\prop_x$ and $\prop_{x'}$:
\begin{equation}\label{eq:zxzxpdef}\begin{aligned}Z_x = \int_y \exp\left(-f(y) - \frac{1}{2\eta}\norm{y - x}_2^2\right)dy,\\ Z_{x'} = \int_y \exp\left(-f(y) - \frac{1}{2\eta}\norm{y - x'}_2^2\right)dy. \end{aligned}\end{equation}

We apply Lemma~\ref{lem:ycloser} and Lemma~\ref{lem:gauint_in_closepts} to bound the ratio of  $Z_x$ and $Z_{x'}$.

\begin{lemma}
	\label{lem:norm_ratio_closepts}
	Let $f$ be $\mu$-strongly convex and $L$-smooth. Let $x, x' \in \Omega_\delta$, for $\Omega_\delta$ defined in \eqref{eq:omegadelta}, and let $\norm{x-x'}_2\leq \Delta$. Then, the normalization constants $Z_x$ and $Z_{x'}$ in \eqref{eq:zxzxpdef} satisfy
	\[
	\frac{Z_x}{Z_{x'}} \leq1.05\exp \left(3LR\Delta + \frac{L\Delta^2}{2}\right).
	\]
\end{lemma}

\begin{proof}
	First, applying Lemma~\ref{lem:gauint_in_closepts} to $Z_x$ and $Z_{x'}$ yields that the ratio is bounded by
	\begin{align*}\frac{Z_x}{Z_{x'}} &\le \frac{\exp\left(-f(y_x) - \frac{1}{2\eta}\norm{y_x - x}_2^2 \right)\left(2\pi\left(\mu + \frac{1}{\eta}\right)^{-1}\right)^{\frac{d}{2}}}{\exp\left(-f(y_{x'}) - \frac{1}{2\eta}\norm{y_{x'} - x}_2^2\right)\left(2\pi\left(L + \frac{1}{\eta}\right)^{-1}\right)^{\frac{d}{2}}} \\
	&\le 1.05\exp\left(f(y_{x'}) - f(y_x) + \frac{1}{2\eta}\left(\norm{y_{x'} - x'}_2^2 - \norm{y_x - x}_2^2\right)\right). \end{align*}
	Here, we used the bound for $\eta^{-1} \ge 32Ld$ that
	\[\left(\frac{L + \frac{1}{\eta}}{\mu + \frac{1}{\eta} }\right)^{d/2} \le 1.05.\]
	Regarding the remaining term, recall $x$, $x'$ both belong to $\Omega_\delta$, and $\norm{x - x'}_2 \le \Delta$. We have 
	\begin{align*}f(y_{x'}) - f(y_x) + \frac{1}{2\eta}\left(\norm{y_{x'} - x'}_2^2 - \norm{y_x - x}_2^2\right) \\
	\le \inprod{\nabla f(y_x)}{y_{x'} - y_x} + \frac{L}{2}\norm{y_{x'} - y_x}_2^2
	+ \frac{1}{2\eta}\inprod{y_{x'} - x' + y_x - x}{y_{x'} - y_x + x - x'} \\
	\le LR\Delta + \frac{L\Delta^2}{2} + \frac{1}{2\eta}\left(\norm{y_x - x}_2 + \norm{y_{x'} - x'}_2\right)\left(\norm{y_{x'} - y_x}_2 + \norm{x' - x}_2\right)\\
	\le LR\Delta + \frac{L\Delta^2}{2} +  \frac{2\eta LR}{2\eta}\left(\norm{y_{x'} - y_x}_2 + \norm{x' - x}_2\right) \le 3LR\Delta + \frac{L\Delta^2}{2}.\end{align*}
	The first inequality was smoothness and expanding the difference of quadratics. The second was by $\norm{\nabla f(y_x)}_2 \le L\norm{y_x - x^*}_2 \le LR$ and $\norm{y_{x'} - y_x}_2 \le \Delta$, where we used the first and second parts of Lemma~\ref{lem:ycloser}; we also applied Cauchy-Schwarz and triangle inequality. The third used the third part of Lemma~\ref{lem:ycloser}. Finally, the last inequality was by the first part of Lemma~\ref{lem:ycloser} and $\norm{x' - x}_2 \le \Delta$.
\end{proof}
We now are ready to prove Lemma~\ref{lem:tv_closepts}.
\restatetrantv*
\begin{proof}
First, by Lemma~\ref{lem:tranleqprop}, it suffices to show $\tvd{\prop_x}{\prop_{x'}} \le \thalf$. Pinsker's inequality states
\[
\tvd{\prop_x }{ \prop_{x'}} \leq\sqrt{\frac{1}{2}d_{\text{KL}}\left(\prop_x,\prop_{x'}\right)},
\]
where $d_{\text{KL}}$ is KL-divergence, so it is enough to show $d_{\text{KL}}\left(\prop_x,\prop_{x'}\right) \le \thalf$. Notice that
\begin{align*}
d_{\text{KL}}\left(\prop_x,\prop_{x'}\right) = \log\left(\frac{Z_{x'}}{Z_x}\right) 
+ \int_y \prop_x(y) \log\left(\frac{\exp\left(-f(y) - \frac{1}{2\eta}\norm{y - x}_2^2\right)}{\exp\left(-f(y) - \frac{1}{2\eta}\norm{y - x'}_2^2\right)}\right) dy.
\end{align*}
By Lemma~\ref{lem:norm_ratio_closepts}, the first term satisfies, for $\Delta \defeq \tfrac{\sqrt{\eta}}{10}$,
\[
\log\left(\frac{Z_{x'}}{Z_x}\right) \leq 3LR\Delta + \frac{L\Delta^2}{2} + \log(1.05).
\]
To bound the second term, we have 
\begin{align*} \int_y \prop_x(y) \log\left(\frac{\exp\left(-f(y) - \frac{1}{2\eta}\norm{y - x}_2^2\right)}{\exp\left(-f(y) - \frac{1}{2\eta}\norm{y - x'}_2^2\right)}\right) dy &= \frac{1}{2\eta}\int_y \prop_x(y) \left(\norm{y - x'}_2^2 - \norm{y - x}_2^2\right)dy\\
&= \frac{1}{2\eta}\int_y \prop_x(y)\inprod{x - x'}{2\left(y - x\right) + \left(x - x'\right)} dy\\
&\le \frac{\Delta^{2}}{2\eta}+\frac{\Delta}{\eta}\left\Vert \int_{y}y \mathcal{P}_{x}(y)dy-x\right\Vert_2.
\end{align*}
Here, the second line was by expanding and the third line was by $\norm{x - x'}_2 \le \Delta$ and Cauchy-Schwarz. By Proposition \ref{prop:min_perturb}, $\left\Vert \int_{y}y\mathcal{P}_{x}(y)dy-x\right\Vert_2 \leq 2\eta LR$, where by assumption the parameters satisfy the conditions of Proposition~\ref{prop:min_perturb}. Then, combining the two bounds, we have
\[
d_{\text{KL}}\left(\prop_x,\prop_{x'}\right) \leq 3LR\Delta + \frac{L\Delta^2}{2} +\frac{\Delta^{2}}{2\eta} +2LR\Delta + \log(1.05)= 5LR\Delta + \frac{L\Delta^2}{2} +\frac{\Delta^{2}}{2\eta} + \log(1.05).
\]
When $\Delta = \tfrac{\sqrt{\eta}}{10}$, $\eta L \le 1$, and $\eta L^2R^2 \leq \thalf$, we have the desired
\[d_{\text{KL}}\left(\prop_x,\prop_{x'}\right)  \le \frac{\sqrt{\eta}LR}{2} + \frac{L\eta}{200} + \frac{1}{200} + \log(1.05)\le \half.\]
\end{proof}

\subsection{Isoperimetry}
\label{ssec:isoperimetry}

In this section, we prove Lemma~\ref{lem:iso}, which asks to show that $\pih_{\Omega_\delta}$ satisfies a log-isoperimetric inequality \eqref{eq:logiso}. Here, we define $\pih_{\Omega_\delta}$ to be the conditional distribution of the $\pih$ $x$-marginal on set $\Omega_\delta$. We recall this means that for any partition $S_1$, $S_2$, $S_3$ of $\Omega_\delta$,
\[\pih_{\Omega_\delta}(S_3) \ge \frac{1}{2\psi}d(S_1, S_2) \cdot \min\left(\pih_{\Omega_\delta}(S_1), \pih_{\Omega_\delta}(S_2)\right) \cdot \sqrt{\log\left(1 + \frac{1}{\min\left(\pih_{\Omega_\delta}(S_1), \pih_{\Omega_\delta}(S_2)\right)}\right)}.\]
The following fact was shown in \cite{ChenDWY19}.
\begin{lemma}[\cite{ChenDWY19}, Lemma 11]\label{lem:logisostrongly}
Any $\mu$-strongly logconcave distribution $\pi$ satisfies the log-isoperimetric inequality \eqref{eq:logiso} with $\psi = \mu^{-\half}$.
\end{lemma}
Observe that $\pi_{\Omega_\delta}$, the restriction of $\pi$ to the convex set $\Omega_\delta$, is $\mu$-strongly logconcave by the definition of $\pi$ \eqref{eq:pidef}, so it satisfies a log-isoperimetric inequality. We now combine this fact with the relative density bounds Lemma~\ref{lem:densityratio} to prove Lemma~\ref{lem:iso}.

\restateiso*
\begin{proof}
Fix some partition $S_1$, $S_2$, $S_3$ of $\Omega_\delta$, and without loss of generality let $\pih_{\Omega_\delta}(S_1) \le \pih_{\Omega_\delta}(S_2)$. First, by applying Corollary~\ref{corr:densityratiodelta}, which shows $\tfrac{d\pi}{d\pih}(x) \in [\thalf, 2]$ everywhere in $\Omega_\delta$, we have the bounds
\[\frac{1}{2}\pi_{\Omega_\delta}(S_1) \leq \pih_{\Omega_\delta}(S_1) \leq 2\pi_{\Omega_\delta}(S_1), \; \frac{1}{2}\pi_{\Omega_\delta}(S_2) \leq \pih_{\Omega_\delta}(S_2) \leq 2\pi_{\Omega_\delta}(S_2),\;\text{and} \;\pih_{\Omega_\delta}(S_3) \geq \frac{1}{2}\pi_{\Omega_\delta}(S_3).\] 
Therefore, we have the sequence of conclusions
\begin{align*}
\pih_{\Omega_\delta}(S_3) &\ge \frac{1}{2}\pi_{\Omega_\delta}(S_3) \\
&\ge \frac{d(S_1, S_2)\sqrt{\mu}}{4} \cdot \min\left(\pi_{\Omega_\delta}(S_1), \pi_{\Omega_\delta}(S_2)\right) \cdot \sqrt{\log\left(1 + \frac{1}{\min\left(\pi_{\Omega_\delta}(S_1), \pi_{\Omega_\delta}(S_2)\right)}\right)}\\
	&\ge \frac{d(S_1, S_2)\sqrt{\mu}}{8} \cdot \pih_{\Omega_\delta}(S_1) \cdot \sqrt{\log\left(1 + \frac{1}{2\pih_{\Omega_\delta}(S_1)}\right)} \\
	&\ge \frac{d(S_1, S_2)\sqrt{\mu}}{16}\cdot \pih_{\Omega_\delta}(S_1) \cdot \sqrt{\log\left(1 + \frac{1}{\pih_{\Omega_\delta}(S_1)}\right)}.
\end{align*}
Here, the second line was by applying Lemma~\ref{lem:logisostrongly} to the $\mu$-strongly logconcave distribution $\pi_{\Omega_\delta}$, and the final line used $\sqrt{\log(1 + \alpha)} \le 2\sqrt{\log(1 + \tfrac{\alpha}{2})}$ for all $\alpha > 0$.
\end{proof}

\subsection{Correctness of $\yor$}
\label{ssec:yoracle}
In this section, we show how we can sample $y$ efficiently in the alternating scheme of the algorithm \sjd, within an extremely high probability region. Specifically, for any $x$ with $\norm{x - x^*}_2 \le \sqrt{\kappa d\log(16\kappa/\delta)} \cdot R_\delta$, where $R_\delta$ is defined in \eqref{eq:omegadelta}, we give a method for implementing 
\[\text{draw } y \propto \exp\left(-f(y) - \frac{1}{2\eta}\norm{y - x}_2^2\right)dy. \]
The algorithm is Algorithm~\ref{alg:yor}, which is a simple rejection sampling scheme.

\begin{algorithm}[ht!]\caption{$\yor(f, x, \eta, \delta)$}
	\label{alg:yor}
	\textbf{Input:} $f$ of form \eqref{eq:pidef} with minimizer $x^*$, $\eta > 0$, $\delta \in [0, 1]$, $x \in \R^d$.\\
	\textbf{Output:} If $\norm{x - x^*}_2 \le \sqrt{\kappa d\log(16\kappa/\delta)}\cdot R_\delta$, return exact sample from distribution with density $\propto \exp(-f(y) - \tfrac{1}{2\eta}\norm{y - x}_2^2)$ (see \eqref{eq:omegadelta} for definition of $R_{\delta}$). Otherwise, return sample within $\delta$ TV from distribution with density $\propto \exp(-f(y) - \tfrac{1}{2\eta}\norm{y - x}_2^2)$.
	\begin{algorithmic}[1]
		\If{$\norm{x - x^*}_2 \le \sqrt{\kappa d\log(16\kappa/\delta)}\cdot R_\delta$}
		\While{\textbf{true}}
		\State Draw $y \sim \Nor(x - \eta \nabla f(x), \eta \id)$
		\State $\tau \sim \text{Unif}[0, 1]$
		\If{$\tau \le  \exp(f(x) + \inprod{\nabla f(x)}{y - x} - f(y))$}
		\State \Return $y$
		\EndIf
		\EndWhile
		\EndIf
		\State \Return Sample $x$ within TV $\delta$ from density $\propto \exp(-f(y) - \tfrac{1}{2\eta}\norm{y - x}_2^2)$ using \cite{ChenDWY19}
	\end{algorithmic}
\end{algorithm}
We first recall properties of rejection sampling (an ``exact'' version of Lemma~\ref{lem:reject_approx_proof} and Corollary~\ref{corr:unbiased_reject_approx}).
\begin{definition}[Rejection sampling]
	\label{def:reject}
	Let $\pi$, $\pih$ be distributions with $\tfrac{d\pi}{dx}(x) \propto p(x)$, $\frac{d\pih}{dx}(x) \propto \hp(x)$. Moreover, suppose for some $C \ge 1$, and all $x \in \R^d$,
	\[\frac{p(x)}{\hp(x)} \le C.\]
	We call the following scheme rejection sampling: repeat independent runs of the following procedure until a point is outputted.
	\begin{enumerate}
		\item Draw $x \sim \pih$.
		\item With probability $\tfrac{p(x)}{C\hat{p}(x)}$, output $x$.
	\end{enumerate}
\end{definition}

\begin{lemma}\label{lem:reject_proof}
Rejection sampling terminates in 
\[\frac{C\int_x \hp(x)dx}{\int_x p(x) dx}.\]
samples from $\pih$ in expectation, and the distribution of the output point is $\pi$.
\end{lemma}
\begin{proof}
The second claim follows from Bayes' rule which implies the conditional density of the output point is proportional to $\hat{p}(x) \cdot \tfrac{p(x)}{C\hat{p(x)}} \propto p(x)$, so the distribution is $\pi$. To see the first claim, the probability any sample outputs is
\[\int_x \frac{p(x)}{C\hat{p}(x)} d\pih(x) = \frac{1}{C}\int_x \frac{\int_x p(x) dx}{\int_x \hp(x)dx} d\pi(x) = \frac{\int_x p(x) dx}{C\int_x \hp(x)dx}.\]
The conclusion follows by independence and linearity of expectation.
\end{proof}

We now prove Lemma~\ref{lem:yor} via a direct application of Lemma~\ref{lem:reject_proof}.

\restateyor*
\begin{proof}
For $\norm{x - x^*}_2 \le \sqrt{\kappa d\log(16\kappa/\delta)}\cdot R_\delta$, $\yor$ is a rejection sampling scheme with 
\[p(y) = \exp\left(-f(y) - \frac{1}{2\eta}\norm{y - x}_2^2\right),\; \hp(y) = \exp\left(-f(x) - \inprod{\nabla f(x)}{y - x} -\frac{1}{2\eta}\norm{y - x}_2^2\right).\]
It is clear that $p(y) \le \hp(y)$ everywhere by convexity of $f$, so we may choose $C = 1$. To bound the expected number of iterations and obtain the desired conclusion, Lemma~\ref{lem:reject_proof} requires a bound on
\begin{equation}\label{eq:phatpratio}\frac{\int_y \exp\left(-f(x) - \inprod{\nabla f(x)}{y - x} -\frac{1}{2\eta}\norm{y - x}_2^2\right)dy}{\int_y \exp\left(-f(y) - \frac{1}{2\eta}\norm{y - x}_2^2\right)dy},\end{equation}
the ratio of the normalization constants of $\hp$ and $p$. First, by Fact~\ref{fact:gaussiandensity},
\[\int_y \exp\left(-f(x) - \inprod{\nabla f(x)}{y - x} -\frac{1}{2\eta}\norm{y - x}_2^2\right)dy = \exp\left(-f(x) + \frac{\eta}{2}\norm{\nabla f(x)}_2^2\right)(2\pi\eta)^{\frac{d}{2}}.\]
Next, by smoothness and Fact~\ref{fact:gaussiandensity} once more,
\begin{align*}
\int_y \exp\left(-f(y) - \frac{1}{2\eta}\norm{y - x}_2^2\right)dy 
&\ge \int_y \exp\left(-f(x) - \inprod{\nabla f(x)}{y - x} - \frac{1 + \eta L}{2\eta}\norm{y - x}_2^2\right) dy \\
&= \exp\left(-f(x) + \frac{\eta}{2(1 + \eta L)}\norm{\nabla f(x)}_2^2\right)\left(\frac{2\pi \eta}{1 + \eta L}\right)^{\frac{d}{2}}.
\end{align*}
Taking a ratio, the quantity in \eqref{eq:phatpratio} is bounded above by
\begin{align*}\exp\left(\left(\frac{\eta}{2} - \frac{\eta}{2(1 + \eta L)}\right)\norm{\nabla f(x)}_2^2\right)\left(1 + \eta L\right)^{\frac{d}{2}} &\le 1.5\exp\left(\frac{\eta^2 L}{2(1 + \eta L)}\norm{\nabla f(x)}_2^2\right)\\
&\le 1.5\exp\left(\frac{\eta^2 L^3}{2}\cdot \left(\frac{16\kappa d^2\log^2(16\kappa/\delta)}{\mu}\right)\right) \le 2.\end{align*}
The first inequality was $(1 + \eta L)^{\frac{d}{2}} \le 1.5$, the second used smoothness and the assumed bound on $\norm{x - x^*}_2$, and the third again used our choice of $\eta$.
\end{proof} 	%

\section{Structural results}
\label{sec:structural}

Here, we prove two structural results about distributions whose negative log-densities are small perturbations of a quadratic, which obtain tighter concentration guarantees compared to naive bounds on strongly logconcave distributions. They are used in obtaining our bounds in Section~\ref{sec:ingredients}, but we hope both the statements and proof techniques are of independent interest to the community. Our first structural result is a bound on normalization constant ratios, used throughout the paper. 

\begin{proposition}
\label{prop:normalizationratio}
	Let $f: \R^d \rightarrow \R$ be  $\mu$-strongly convex with minimizer $x^*$, and let $\lambda > 0$. Then,
	\[\frac{\int \exp(-f(x)) dx}{\int \exp\left(-f(x) - \frac{1}{2\lambda}\norm{x - x^*}_2^2\right)dx} \le \left(1 + \frac{1}{\mu\lambda}\right)^{\frac{d}{2}}. \]
\end{proposition}
\begin{proof}
	Define the function 
	\[R(\alpha) \defeq \frac{\int \exp\left(-f(x) - \frac{1}{2\lambda\alpha}\norm{x - x^*}_2^2\right) dx}{\int \exp\left(-f(x) - \frac{1}{2\lambda}\norm{x - x^*}_2^2\right)dx}. \]
	Let $d\pi_\alpha(x)$ be the density proportional to $\exp\left(-f(x) - \tfrac{1}{2\lambda\alpha}\norm{x - x^*}_2^2\right)dx$. We compute
	\begin{align*}\frac{d}{d\alpha} R(\alpha) &= \int \frac{\exp\left(-f(x) - \frac{1}{2\lambda\alpha}\norm{x - x^*}_2^2\right)}{\int \exp\left(-f(x) - \frac{1}{2\lambda}\norm{x - x^*}_2^2\right)dx} \frac{1}{2\lambda\alpha^2} \norm{x - x^*}_2^2 dx \\
	&= \frac{R(\alpha)}{2\lambda\alpha^2} \int \frac{\exp\left(-f(x) - \frac{1}{2\lambda\alpha}\norm{x - x^*}_2^2\right)\norm{x - x^*}_2^2}{\int \exp\left(-f(x) - \frac{1}{2\lambda\alpha}\norm{x - x^*}_2^2\right)dx}   dx \\
	&= \frac{R(\alpha)}{2\lambda\alpha^2}\int \norm{x - x^*}_2^2 d\pi_\alpha(x) \le \frac{R(\alpha)}{2\alpha} \cdot \frac{d}{\mu\lambda\alpha + 1}. \end{align*}
	Here, the last inequality was by Fact~\ref{fact:distxstar}, using the fact that the function $f(x) + \tfrac{1}{2\lambda\alpha}\norm{x - x^*}_2^2$ is $\mu + \tfrac{1}{\lambda\alpha}$-strongly convex. Moreover, note that $R(1) = 1$, and
	\[\frac{d}{d\alpha}\log\left(\frac{\alpha}{\mu\lambda\alpha + 1}\right) = \frac{1}{\alpha} - \frac{\mu\lambda}{\mu\lambda\alpha + 1}=  \frac{1}{\mu\lambda\alpha^2 + \alpha}.\]
	Solving the differential inequality
	\[\frac{d}{d\alpha} \log(R(\alpha)) = \frac{dR(\alpha)}{d\alpha} \cdot \frac{1}{R(\alpha)} \le \frac{d}{2} \cdot \frac{1}{\mu\lambda\alpha^2 + \alpha}, \]
	we obtain the bound for any $\alpha \ge 1$ (since $\log(R(1)) = 0$)
	\[\log(R(\alpha)) \le \frac{d}{2}\log\left(\frac{\mu\lambda\alpha + \alpha}{\mu\lambda\alpha + 1}\right) \implies R(\alpha) \le \left(\frac{\mu\lambda\alpha + \alpha}{\mu\lambda\alpha + 1}\right)^{\frac{d}{2}} \le \left(1 + \frac{1}{\mu\lambda}\right)^{\frac{d}{2}}.\]
	Taking a limit $\alpha \rightarrow \infty$ yields the conclusion.
\end{proof}

Our second structural result uses a similar proof technique to show that the mean of a bounded perturbation $f$ of a Gaussian is not far from its mode, as long as the gradient of the mode is small. We remark that one may directly apply strong logconcavity, i.e.\ a variant of Fact~\ref{fact:distxstar}, to obtain a weaker bound by roughly a $\sqrt{d}$ factor, which would result in a loss of $\Omega(d)$ in the guarantees of Theorem~\ref{thm:mainclaim}. This tighter analysis is crucial in our improved mixing time result. 

Before stating the bound, we apply Fact~\ref{fact:convexshrink} to the convex functions $h(x) = (\theta^\top x)^2$ and $h(x) = \norm{x}_2^4$ to obtain the following conclusions which will be used in the proof of Proposition~\ref{prop:min_perturb}.

\begin{corollary}
	\label{corr:slcmomentbounds}
	Let $\pi$ be a $\mu$-strongly logconcave density. Then,
	\begin{enumerate}
		\item $\E_{\pi}[(\theta^\top(x - \E_\pi[x]))^2] \le \mu^{-1}$, for all unit vectors $\theta$.
		\item $\E_{\pi}[\norm{x - \E_\pi[x]}_2^4] \le 3d^2\mu^{-2}$.
	\end{enumerate}
\end{corollary}

\begin{proposition}
	\label{prop:min_perturb}
	Let $f: \R^d \rightarrow \R$ be $L$-smooth and convex with minimizer $x^*$, let $x \in \R^d$ with $\norm{x - x^*}_2 \le R$, and let $d\pi_\eta(y)$ be the density proportional to $\exp\left(-f(y) - \tfrac{1}{2\eta}\norm{y - x}_2^2\right)dy$. Suppose that $\eta \leq \min\left(\tfrac{1}{2L^2R^2},\tfrac{ R^2}{400d^2}\right)$. Then,
	\[\norm{\E_{\pi_\eta}[y] - x}_2 \le 2\eta LR. \]
\end{proposition}
\begin{proof}
	Define a family of distributions $\pi^\alpha$ for $\alpha \in [0, 1]$, with
	\[d\pi^\alpha(y) \propto \exp\left(-\alpha\left(f(y) - f(x) - \inprod{\nabla f(x)}{y - x}\right) - f(x) - \inprod{\nabla f(x)}{y - x} - \frac{1}{2\eta}\norm{y - x}_2^2\right)dy. \]
	In particular, $\pi^1 = \pi_\eta$, and $\pi^0$ is a Gaussian with mean $x - \eta\nabla f(x)$. We define $\bya \defeq \E_{\pi_\alpha}[y]$, and
	\[\sya \defeq \argmin_y \left\{\alpha\left(f(y) - f(x) - \inprod{\nabla f(x)}{y - x}\right) + f(x) + \inprod{\nabla f(x)}{y - x} + \frac{1}{2\eta}\norm{y - x}_2^2\right\}.\]
	Define the function $D(\alpha) \defeq \norm{\bya - x}_2$, such that we wish to bound $D(1)$. First, by smoothness
	\[D(0) = \norm{\E_{\pi_0}[y] - x}_2 = \norm{\eta\nabla f(x)}_2 \le \eta LR.\]
	Next, we observe
	\[\frac{d}{d\alpha}D(\alpha) = \inprod{\frac{\bya - x}{\norm{\bya - x}_2}}{\frac{d\bya}{d\alpha}} \le \norm{\frac{d\bya}{d\alpha}}_2. \]
	In order to bound $\norm{\tfrac{d\bya}{d\alpha}}_2$, fix a unit vector $\theta$. We have
	\begin{equation}
	\label{eq:bigderivbound}
	\begin{aligned}\inprod{\frac{d\bya}{d\alpha}}{\theta} &= \frac{d}{d\alpha}\inprod{\int (y - x) d\pi^\alpha(y)}{\theta} \\
	&= \int \inprod{y - x}{\theta} (f(x) + \inprod{\nabla f(x)}{y - x} - f(y)) d\pi^\alpha(y) \\
	&\le \sqrt{\int (\inprod{y - x}{\theta})^2 d\pi^\alpha(y)}\sqrt{\int (f(x) + \inprod{\nabla f(x)}{y - x} - f(y))^2 d\pi^\alpha(y)}\\
	&\le \sqrt{\int (\inprod{y - x}{\theta})^2 d\pi^\alpha(y)}\sqrt{\int \frac{L^2}{4}\norm{y - x}_2^4 d\pi^\alpha(y)}.
	\end{aligned}
	\end{equation}
	The third line was Cauchy-Schwarz and the last line used smoothness and convexity, i.e.
	\begin{align*}
	-\frac{L}{2}\norm{y - x}_2^2 \le f(x) + \inprod{\nabla f(x)}{y - x} - f(y) \le 0.
	\end{align*}
	We now bound these terms. First,
	\begin{equation}
	\label{eq:firsttermcs}
	\begin{aligned}
	\int (\inprod{y - x}{\theta})^2d\pi^\alpha(y) &\le 2\int(\inprod{y -\bya}{\theta})^2d\pi^\alpha(y) +    2\int(\inprod{\bya - x}{\theta})^2d\pi^\alpha(y) \\
	&\le 2\eta + 2\norm{\bya - x}_2^2 = 2\eta + 2D(\alpha)^2.
	\end{aligned}
	\end{equation}
	Here, we applied the first part of  Corollary~\ref{corr:slcmomentbounds}, as $\pi^\alpha$ is $\eta^{-1}$-strongly logconcave, and the definition of $D(\alpha)$. Next, using for any $a, b \in \R^d$, $\norm{a + b}_2^4 \le (\norm{a}_2 + \norm{b}_2)^4 \le 16\norm{a}_2^4 + 16\norm{b}_2^4$, we have
	\begin{equation}
	\label{eq:secondtermcs}
	\begin{aligned}
	\int \frac{L^2}{4}\norm{y - x}_2^4 d\pi^\alpha(y) &\le \int 4L^2\norm{y - \bya}_2^4 d\pi^\alpha(y) + \int 4L^2\norm{x - \bya}_2^4 d\pi^\alpha(y) \\
	&\le 12L^2 d^2 \eta^2 + 4L^2 D(\alpha)^4.
	\end{aligned}
	\end{equation}
	Here, we used the second part of Corollary~\ref{corr:slcmomentbounds}. Maximizing \eqref{eq:bigderivbound} over $\theta$, and applying \eqref{eq:firsttermcs}, \eqref{eq:secondtermcs},
	\begin{align}
	\label{eq:deri_D_bound}
	\frac{d}{d\alpha}D(\alpha) \le \norm{\frac{d\bya}{d\alpha}}_2 &\le \sqrt{8L^2(\eta + D(\alpha)^2)(3d^2\eta^2 + D(\alpha)^4)} \notag \\
	& \leq 4L(\sqrt{\eta} +D(\alpha))\cdot \max(2\eta d, D(\alpha)^2).
	\end{align}
	Assume for contradiction that $D(1) > 2\eta LR$, violating the conclusion of the proposition. By continuity of $D$, there must have been some $\bal \in (0, 1)$ where $D(\bal) = 2\eta LR$, and for all $0 \leq \alpha < \bal$, $D(\alpha) < 2\eta L R$. By the mean value theorem, there then exists $0 \le \hal \le \bal$ such that 
	\[
	\frac{dD(\hal)}{d\alpha} = \frac{D(\bal) - D(0)}{\bal} > \eta LR.
	\]
	On the other hand, by our assumption that $2\eta L^2R^2\leq 1$, for any $d \ge 1$ it follows that
	\[2\eta d \geq 4\eta^2 L^2R^2 > D(\hal)^2,\; \sqrt{2\eta} \geq 2\eta L R > D(\hal).\]
	Then, plugging these bounds into \eqref{eq:deri_D_bound} and using $\sqrt{\eta} + D(\hal) \le \tfrac{5}{2}\sqrt{\eta}$ as $\sqrt{2} \le \tfrac{3}{2}$,
   \[
   \frac{d}{d\alpha}D(\hal) \le 4L \cdot \frac 5 2\sqrt{\eta} \cdot 2\eta d = 20\sqrt{\eta}\frac{d}{R} \cdot \eta LR \leq  \eta LR.
   \]
   We used $\eta \leq \tfrac{R^2}{400d^2}$ in the last inequality. This is a contradiction, implying $D(1) \le 2\eta LR$.
\end{proof}

\section{Approximation tolerance}
\label{app:approximate}

We briefly discuss the tolerance of our algorithm to approximation error in two places: computation of the point $x^*$, and implementation of the restricted Gaussian oracle for the composite function $g$. 

\paragraph{Inexact minimization.} Standard methods such as the FISTA method of \cite{BeckT09} imply that under access to gradient queries to $f$ and a proximal oracle for $g$, we can find the minimizer to inverse polynomial accuracy in problem parameters (measured by Euclidean distance to the true minimizer) with negligible increase in runtime. By expanding the radii $R_\delta$ in the definition of the sets $\Omega_\delta$ by a constant factor, this accomodates tolerance to inexact minimization and only affects all bounds throughout the paper by constants.

\paragraph{Inexact oracle implementation.} Similarly, our algorithm is tolerant to total variation error inverse polynomial in problem parameters for the restricted Gaussian oracle for $g$. To see this, we pessimistically handled the case where the sampler $\yor$ for a quadratic restriction of $f$ resulted in total variation error in the proof of Proposition~\ref{prop:sjdguarantee}, assuming that the error was incurred in every iteration. By accounting for similar amounts of error in calls to $\oracle$, the bounds in our algorithm are only affected by constants. 	%

\section{Additional details for Section~\ref{sec:experiments}}
\label{app:experiment_details}

We provide additional details for our experiments here.

\textbf{Correctness verification.} 
We verify the correctness of our algorithm against the output of na\"ive rejection sampling (accepting samples in $O$). The rejection sampling algorithm generates samples from the unrestricted Gaussian distribution and rejects the samples falling outside the chosen orthant, so the resulting distribution follows the target distribution exactly. Due to the curse of dimensionality, it is only possible to do rejection sampling in low dimensions; we choose dimension $d = 10$. The Gaussian distribution is randomly generated with a dense covariance matrix and mean $m$, with $m_i \sim \text{Unif}[-0.5, 0.5]$ for each coordinate $i$. We plot the 2D histograms of $N = 3000$ samples projected on 5 pairs of random directions in Figure~\ref{fig:hist}. In running our algorithm for this experiment, we choose $\eta = 0.01 $ and $K = 500$ in \sjd. 

\begin{figure}[h]
	\centering
	\includegraphics[width=\linewidth]{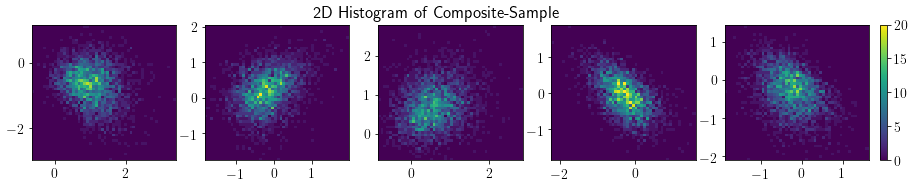}
	\includegraphics[width=\linewidth]{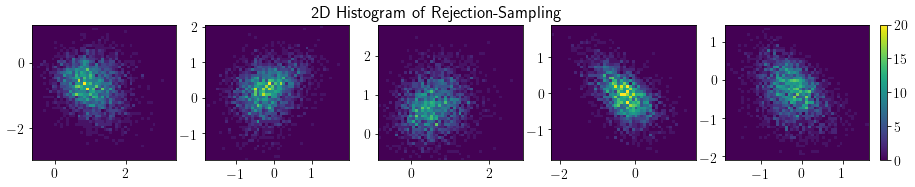}
	\caption{2D histograms of $\csg$ and rejection sampling under 5 pairs of random projections for $d = 10$.}
	\label{fig:hist}
\end{figure}

\textbf{Additional autocorrelation plots.} Figure~\ref{fig:auto_apdx} shows autocorrelation plots of the trajectories of our algorithm and hit-and-run with iteration counts $K = 20000$  and $K = 500000$. The parameter choices used in this experiment are stated in Section~\ref{sec:experiments}. 
\begin{figure}[h]
	\begin{subfigure}{.5\textwidth}
		\centering
		\includegraphics[scale=0.5 ]{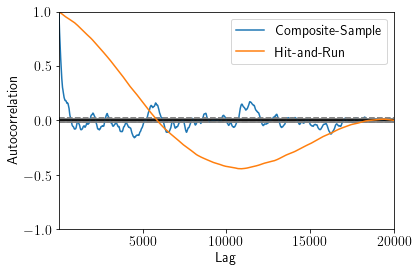}

	\end{subfigure}
	\begin{subfigure}{.5\textwidth}
		\includegraphics[scale=0.5]{autocorrelation2.png}
		
	\end{subfigure}
	\caption{Autocorrelation plot of composite-sample and hit-and-run for $d = 500$.}
	\label{fig:auto_apdx}

\end{figure}

 	\end{appendix}
	
\end{document}